%% file: iclr2027_conference.tex
\documentclass{article} 
\usepackage{iclr2027_conference,times}

\input{math_commands.tex}

\usepackage{enumitem}
\usepackage{hyperref}
\usepackage{url}
\usepackage{amsthm}
\newtheorem{proposition}{Proposition}
\usepackage{graphicx} 
\usepackage{amssymb}
\usepackage{wrapfig}
\usepackage{graphicx}
\usepackage{booktabs}
\usepackage{amsthm}
\hypersetup{hidelinks}
\theoremstyle{definition}
\newtheorem{definition}{Definition}
\theoremstyle{plain}
\title{Scene Reconstruction using Ubiquitous Wireless Signals via Diffusion Posterior Inference over physics-grounded representations}

\author{Vahid Yazdnian \& Yasaman Ghasempour \\
Princeton University, Department of Electrical and Computer Engineering \\
\texttt{\{vahidyazdnian,ghasempour\}@princeton.edu} \\
}

\iclrfinalcopy 
\begin{document}

\maketitle
\lhead{Scene Reconstruction Using Ubiquitous Wireless Signals}

\vspace{-8mm}
\begin{abstract}
\vspace{-1mm}
\vspace{-2mm}
Knowing the layout of the physical world matters for many applications, yet the sensors we rely on for it (e.g., cameras and LiDAR) can fail in darkness, in adverse weather, or behind walls. Wireless signals do not face these limitations: every link between a base station and a user is shaped by the buildings it passes, reflects from, or is blocked by, leaving information about the surrounding scene embedded in the wireless channel measured at the endpoints.
However, mapping channel measurements to the environment is a severely ill-posed problem: a measurement records only the structures that interact with the transmitted electromagnetic waves along its propagation paths, and even these observations can often be explained by multiple plausible scenes. We therefore formulate wireless scene reconstruction as posterior inference and develop a hierarchical framework with two complementary stages to approach this problem. A wireless foundation model, pretrained to respect propagation physics, transforms each channel measurement into a representation of the scene structure that shaped the observation. A diffusion model then jointly integrates representations from multiple measurements to model and sample the posterior over building occupancy, with consensus over samples producing the final reconstruction. 
We demonstrate that the framework can accurately recover building footprints from wireless channel measurements alone and show that its performance depends critically on both design components: 
physics-grounded representations expose scene-relevant information, while joint posterior inference coherently integrates complementary observations.

\end{abstract}

\vspace{-6mm}
\section{Introduction}
\vspace{-3mm}
Every wireless link is already an unintentional measurement of its surroundings. A signal that leaves a base station and arrives at a receiver has encountered several components in an environment: blocked by one building, reflected off another, diffracted around a corner. The environment is, in effect, written into the wireless channel measured at the receiver. If that writing could be read back, the radio infrastructure already deployed at scale could be repurposed as a sensor network, mapping the world through occluded regions, in darkness, and in severe weather, where cameras and lidar degrade or fail outright~\citep{ zhao2018throughwall, dodds2026wave}. It would also keep maps current when they change fastest: after an earthquake or a flood, the layout of a neighborhood can change in minutes, and the network is often the only sensor still running. Site-specific digital twins, which today require expensive surveys or hand-built geometry, could also be constructed and kept current from the traffic a network already carries~\citep{hoydis2023sionna, li2024geo2sigmap, testolina2024bostontwin}. The knowledge of the propagation environment can also be used to optimize communication network throughput~\citep{zhao2023nerf2, alkhateeb2023isac}, and could advance integrated communication and sensing, one of the main pillars of next-generation networks~\citep{akyildiz20206g, wang2024genai_isac}. 

Recovering geometry from wave measurements is well studied in acoustics, where room shape is inferred from resolvable echoes and geometric propagation models~\citep{dokmanic2013echoes, elbaba20183d}; recent learning-based approaches relax some of these assumptions and predict floorplans directly~\citep{yeon2024rginet, yeon2024echoscan}, but still rely on echo structures that can be attributed to individual surfaces. In the RF domain, scene reconstruction has primarily relied on dedicated radar sensing systems~\citep{casmin2025survey,borts2026carafe}, recovering occupancy grids or floor plans from millimeter-wave and ultra-wideband reflections using geometric inference, motion-synthesized aperture, structural priors, or learned generative models~\citep{lu2020millimap,sun2022indoor,zhou2025rise,lai2026non}. Recent approaches include diffusion-based indoor scene completion from sparse radar heatmaps~\citep{zhou2025rise} and reconstruction around corners from multi-bounce returns collected by a mobile radar~\citep{lai2026non}. These approaches require dedicated sensing hardware and often movement through the environment, incurring additional hardware and deployment costs.


In this paper, we aim to use the communication signals from already deployed infrastructure for scene reconstruction. The challenge is that extracting environmental components out of a raw wireless channel measurement without any assumptions is a severely ill-posed inverse problem, and the ill-posedness is structural rather than incidental. A measurement is shaped only by the structures its propagation paths interact with, so the likelihood is exactly invariant to the rest of the scene. Even within the observed region, the geometry-to-channel mapping is many-to-one: finite bandwidth and aperture limit the resolution of path delays and angles, so scatterers displaced within those limits yield indistinguishable measurements. Compounding this, a structure impacts the channel only through the \emph{role} it plays for a given link --- reflecting, blocking, or leaving the signal untouched --- and that role varies with where the structure sits relative to the transmitter and receiver. Different occupancy patterns whose structures play the same roles for a given link would therefore be indistinguishable. Hence, a single measurement admits not one scene but a family of them.

In this work, we formulate scene reconstruction as posterior inference over occupancy: the solution is a distribution over scenes, and the observations determine which parts of it are pinned down and which remain free. To our knowledge, this is the first treatment of wireless scene reconstruction that assumes nothing about the number, shape, convexity, or location of scatterers.

Learning this posterior directly from raw channel measurements is impractical.
A single decoder would have to solve two difficult problems at once: work out from finite data how propagation physics maps geometry into a channel, and at the same time learn to generate plausible scenes from a combinatorially large space. We therefore adopt a \emph{hierarchical} design that separates these burdens.  First, a multi-modal wireless foundation model, pretrained to capture electromagnetic propagation physics and align channel observations with propagation-relevant scene structure, maps each measurement to a representation of the scene properties that could have produced it. Second, representations from multiple observations of the same environment are fused to condition a diffusion decoder, which performs posterior inference over occupancy and generates scenes consistent with the observed channels. Finally, we aggregate the sampled scenes through per-cell consensus to obtain the reconstructed occupancy map.

We evaluate reconstruction of arbitrary scene geometries against natural alternatives to each level of the hierarchy on $\sim$700 environments held out from all stages of training, spanning diverse spatial layouts and structural configurations, from dense urban blocks to open areas with isolated structures. Specifically, we show through controlled comparisons that both stages of the hierarchy are essential: independently reconstructing and aggregating per-observation scene estimates markedly underperforms, while conditioning the decoder directly on raw channel measurements collapses. By varying the input representation and posterior inference mechanism independently, we isolate the contribution of each stage. Together, these experiments demonstrate that physics-grounded representation learning followed by generative posterior inference provides the complementary capabilities needed to make reconstruction from sparse wireless observations tractable.

\textbf{Contributions.}
\vspace{-2mm}
\begin{itemize}[leftmargin=1em, itemsep=0pt, topsep=0pt]
\item \textbf{Formulation and theory.} We formulate wireless scene reconstruction as probabilistic posterior inference over occupancy with no parametric prior; characterize its ill-posedness and the resulting recoverable set; and derive finite-sample bounds showing that a consensus estimator over posterior samples is Bayes-optimal per cell and converges to the true scene on the recoverable set as the number of measurements grows.
\item \textbf{Physics-grounded representations.} We show that pretraining a wireless foundation model and aligning its representation with propagation-relevant scene structures yields per-observation embeddings from which the scene occupancy distribution can be recovered. Each embedding captures the portion of the environment that shaped the corresponding measurement while expressing calibrated uncertainty elsewhere.
\item \textbf{Hierarchical generative reconstruction.} We design a diffusion decoder that conditions on the set of per-observation representations and samples the occupancy posterior, and propose a consensus readout that fuses many viewpoints into a single reconstruction without assuming anything about scene structure.

\end{itemize}

\vspace{-4mm}
\section{Preliminaries and Problem Formulation}
\label{sec:formulation}
\vspace{-2mm}
\subsection{Wireless propagation}
\label{sec:propagation}
\vspace{-2mm}
A signal transmitted from a base station reaches a receiver along a set of propagation paths determined by the surrounding environment. Each path is a sequence of interactions with the physical structures of the scene: it may travel directly, reflect specularly off a building, diffract at an edge, or be blocked. For a receiver at position $\mathbf{p}$, propagation through a scene $\mathbf{s}$  thats $\mathcal{P}(\mathbf{s}, \mathbf{p})= \bigl\{(\alpha_\ell,\ \tau_\ell,\ \theta_\ell,\ \phi_\ell)\bigr\}_{\ell=1}^{L}$, where $\alpha_\ell$ is the complex gain, $\tau_\ell$ the delay, and $\theta_\ell,\phi_\ell$ the angles of arrival and departure. Each of these parameters depends on scene geometry: delays are set by the lengths of the path segments, angles by the orientation of the surfaces the wave meets, and gains by free-space loss and the reflection coefficients of the structures encountered.

The receiver does not observe the paths individually. For a wideband multiple-input multiple-output (MIMO) link with $N_{tx}$ transmit antennas, $N_{rx}$ receive antennas, and $N_{sc}$ subcarriers, it observes their superposition in the measured channel $H \in \mathbb{C}^{N_{tx}\times N_{rx}\times N_{sc}}$, whose entry for transmit/receive antenna $t/r$ and subcarrier frequency $f$ is
\vspace{-2mm}
\begin{equation} 
H[t,r,f] \;=\; \sum_{\ell \in \mathcal{P}(\mathbf{s},\mathbf{p})} \alpha_\ell\, a_{r}(\theta_\ell)\, a_{t}^{*}(\phi_\ell)\, e^{-j 2\pi f \tau_\ell}, 
\vspace{-2mm}
\label{eq:pathsum} \end{equation}
with $a_t(\cdot)$ and $a_r(\cdot)$ the array responses (Appendix~\ref{app:forward}). This composition defines a forward operator $H = \mathcal{F}(\mathbf{s}, \mathbf{p})$ mapping scene and link geometry to the observed channel: the scene is written into the channel through the path parameters, in a form the receiver observes only in superposition. Importantly, not all parts of the environment participate in the measurement. Only structures a path interacts with, allowing free propagation, blocking the direct ray, or serving as reflection or diffraction points, enter the path set and thus the channel. We call this region the \emph{propagation-relevant set} of an observation; structures outside it leave the observation exactly unchanged.

\vspace{-2mm}
\subsection{The inverse problem}
\label{sec:inverse}
\vspace{-2mm}
\textbf{Objective.} We observe a scene through the wireless links a deployed network already maintains. A base station serves $K$ users, where user $k$ at position $\mathbf{p}_k$ contributes a channel measurement $H_k$, giving the observation set $\mathcal{O} = \bigl\{(H_k,\mathbf{p}_k)\bigr\}_{k=1}^{K}$. Denoting $a_k$ for the propagation-relevant set of observation $k$, the \emph{observable set} $\mathcal{A}=\bigcup_k a_k$ is the region of the scene influencing at least one measurement. We represent the environment as a two-dimensional occupancy grid over a square region with spatial resolution $\Delta R$, $\mathbf{s}\in\{0,1\}^{N\times N}$, where $s_i=1$ indicates that cell $i$ contains a structure, and seek to recover $\mathbf{s}$ from $\mathcal{O}$, inverting $\mathcal{F}$, without parametric assumptions on the number, shape, convexity, or location of the structures.

\textbf{Why the problem is ill-posed.} Under the standard Hadamard framework~\citep{bakushinsky2012ill}, the problem is ill-posed because the observations do not uniquely determine the scene, despite a feasible solution existing. Non-uniqueness arises from three sources (Appendix~\ref{app:illposed}). First, structures outside $\mathcal{A}$ leave every measurement unchanged, so the likelihood is exactly invariant to them: flipping any combination of those cells yields a scene explaining the observations equally well, and the number of such scenes is exponential in $|\mathcal{A}^{c}|$. Second, even inside $\mathcal{A}$, the forward map is many-to-one, since finite bandwidth and finite aperture resolve delays and angles only to a limit, so scatterers displaced within those limits yield indistinguishable measurements. Third, a structure enters the channel only through the \emph{role} it plays for a given link --- reflecting, blocking, or leaving the signal untouched --- and that role varies with where it sits relative to transmitter and receiver: the same building reflects for one user, blocks another, and is invisible to a third. Scenes whose structures play the same roles for every link are indistinguishable at any resolution, so unlike the second source, this degeneracy is not removed by better hardware. Multiple observations narrow the ambiguity by intersecting the families of scenes each admits, and additional users and reflected paths enlarge $\mathcal{A}$, but the intersection does not collapse to a single scene: the invariance outside $\mathcal{A}$ persists, and the resolution and role ambiguities remain within it.

\begin{figure}[t]
\vspace{-1mm}
    \centering
    \includegraphics[width=0.95\linewidth]{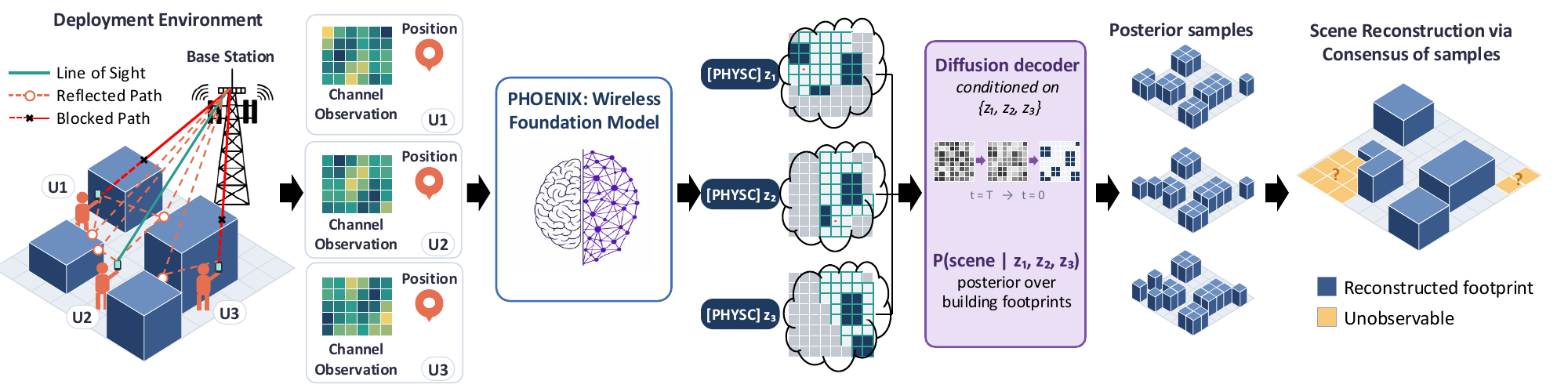}
    \vspace{-4mm}
    \caption{
       \small Overview of the proposed framework for wireless scene reconstruction using diffusion posterior inference over wireless foundation model representations.
    }
    \label{fig:overview}
    \vspace{-6mm}
\end{figure}

\vspace{-4mm}
\subsection{A probabilistic formulation}
\label{sec:probabilistic}
\vspace{-3mm}
Because the data do not determine a unique scene, we do not seek one. We take the solution object to be the posterior over scenes $p(\mathbf{s}|\mathcal{O}) \propto p(\mathcal{O}|\mathbf{s}) p(\mathbf{s})$, whose structure is dictated by the physics of wave propagation. Since structures outside $\mathcal{A}$ leave every measurement unchanged, the likelihood is invariant to them, and hence factorizing the prior gives
\vspace{-1mm}
\begin{equation} 
p(\mathbf{s} \mid \mathcal{O}) \;=\; \underbrace{p(\mathbf{s}_{\mathcal{A}} \mid \mathcal{O})}_{\text{extracted from the data}} \;\cdot\; \underbrace{p(\mathbf{s}_{\mathcal{A}^{c}} \mid \mathbf{s}_{\mathcal{A}})}_{\text{supplied by the prior}}. 
\vspace{-3mm}
\label{eq:factorization} \end{equation}

The observation set does not appear in the second factor: it acts on the observable part of the scene and leaves the rest to the prior, conditioned on what the data did determine. The per-cell marginals $\pi_i = p(s_i = 1 | \mathcal{O})$ make this explicit. For $i \in \mathcal{A}$ the measurements can drive $\pi_i$ toward $0$ or $1$ as evidence accumulates; for $i \notin \mathcal{A}$ the marginal reduces to $\pi_i = p(s_i = 1 \mid \mathbf{s}_{\mathcal{A}})$, which the measurements cannot move, and its variance converges to a prior-conditional floor rather than to zero. A deterministic regressor, which must commit a value to every cell, cannot represent this distinction and produces confident predictions on cells the measurements never constrained; a posterior can, which is what makes it the appropriate object here. Since the likelihood $p(\mathcal{O}|\mathbf{s})$ is defined only implicitly through ray tracing and is unavailable in closed form, we learn the posterior directly from paired physically consistent data $(\mathbf{s}, \mathcal{O}) \sim \mathcal{D}$, as described next.

\vspace{-4mm}
\section{Design and Methodology}
\label{sec:method}
\vspace{-3mm}
Our objective is to learn the posterior over scenes given the measurements a network already collects. The difficulty is not only that the posterior is high-dimensional, but that a model trained directly on raw channels would have to solve two problems at once. It would have to discover, from finite data, how geometry writes itself into a channel --- which structures produce which delays, which angle of arrivals, which signal blockage --- and simultaneously learn to generate a coherent scene consistent with that mapping. The first of these is the physics of propagation and is shared across every link; the second is the inference problem specific to a given set of observations. Entangling them enforces supervision on rediscovering the underlying physics, and leaves the generative model to place mass over a combinatorially large space of grids with no guidance about where the plausible scenes lie. We therefore propose a hierarchical framework, illustrated in Fig.~\ref{fig:overview}, that separates these two burdens. The first stage learns the shared physical characteristics once: a multi-modal foundation model is pretrained to aggregate channel and positional information into a representation that is explicitly aligned, during pretraining, with the scene structure that shaped the observation. Each measurement is thereby encoded into a summary of the scene properties that could have produced it. The second stage performs the inference: the encoded representations of many observations of the same environment are fused as the conditioning set of a diffusion model that learns the posterior over occupancy, and a consensus over its samples yields the reconstructed scene. Because conditioning is expressed in a representation space that is aligned with geometrical properties rather than raw channel information, the generative model just resolves which scenes are consistent with the measurements rather than first inverting propagation.

\vspace{-4mm}
\subsection{Physics-grounded representation learning}
\label{sec:fm}
\vspace{-2mm}
In this section, we discuss the design and implementation of PHOENIX (a \underline{P}hysics-grounded \underline{H}eterogeneous \underline{O}bservation \underline{E}mbedding for \underline{N}eural \underline{I}nference across wireless conte\underline{X}ts), a multi-modal foundation model that learns a unified representation of the physical information governing wireless propagation. The overall pipeline of PHOENIX is illustrated in Fig.~\ref{fig:FM}.

At its core, PHOENIX employs a single-stream multi-modal transformer-based feature embedding framework to learn global, robust representations of the wireless channel by capturing cross-modality associations among the input data covering the wireless medium, the transmitter-receiver positioning, and the spatial-temporal characteristics of the received signal. Specifically, the channel tensor in the input is converted into amplitude and phase and tokenized over antenna pairs and neighboring subcarriers; the height map is partitioned into non-overlapping patches; and the receiver location is projected to a single token. Each token additionally encodes its physical origin through antenna/subcarrier indices or normalized spatial coordinates, together with a modality embedding. A dedicated [PHYSC] token aggregates information across modalities under an asymmetric attention mask: it attends to all modality tokens but is excluded as a key from the remaining tokens. The \texttt{[PHYSC]} token is designed to aggregate information across modalities into a shared, physically grounded representation that reflects the underlying wireless propagation structure. The key question is therefore how to ensure that this representation remains consistent with the observed wireless propagation and underlying physical correspondence across modalities. Architectural and tokenization details are provided in Appendix~\ref{app:arch}.

\begin{figure}[t]
    \centering
    \includegraphics[width=0.9\linewidth]{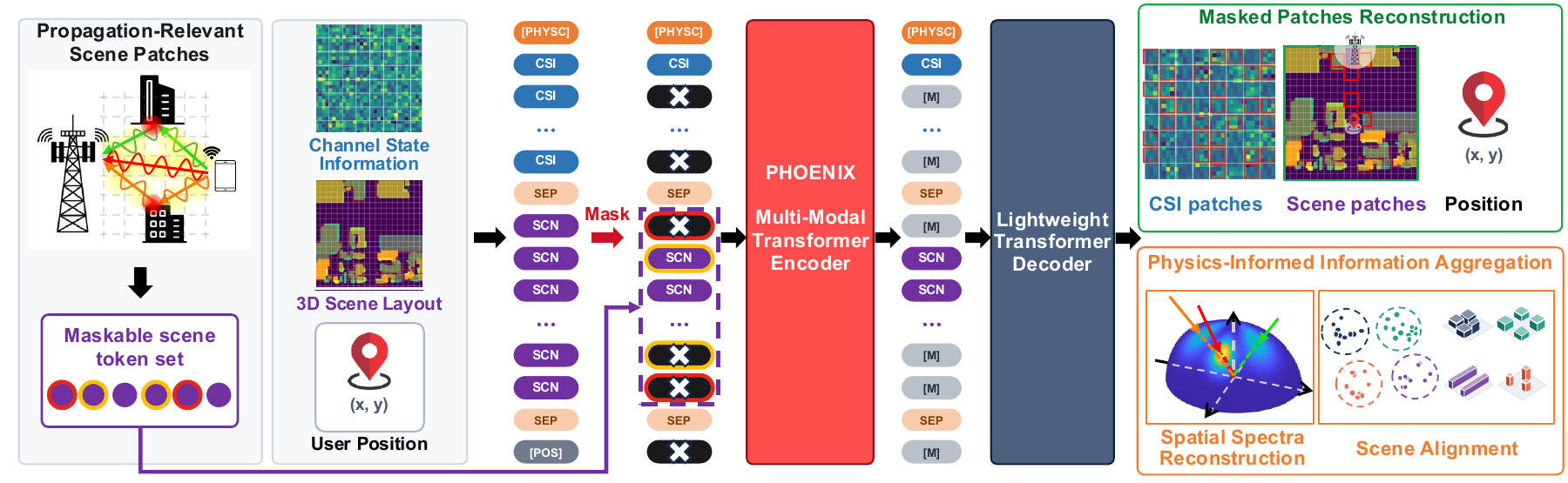}
    
    \vspace{-4mm}
    \caption{\small PHOENIX self-supervised pretraining: enabling physics-grounded representation learning over physical-layer wireless systems. 
    }
    \label{fig:FM}
    \vspace{-6mm}
\end{figure}

\textbf{Physics-guided self-supervised pretraining.} What makes the resulting representation physically meaningful is the pretraining objective. Conventional contrastive image--text style objectives~\citep{radford2021learning} are a poor fit here, as different scenes and user positions can produce similar channels through different combinations of propagation paths, so matching and non-matching channel--scene pairs are not well defined. We instead pretrain with two complementary mechanisms.

The first mechanism is \emph{bidirectional reconstruction}, which learns cross-modal physical associations by reconstructing masked information in one modality from the others: masked complex channel values are recovered from the remaining channel tokens together with scene and location information, while masked regions of the building map are recovered from the observed channel. Crucially, scene masking is not applied uniformly. Because a wireless observation constrains only the environment reached by its propagation paths, asking the model to infer regions that left no trace in the channel would train it to hallucinate unsupported geometry. We therefore identify the scene patches intersected by the collective propagation paths of a given link and mask \emph{only within} this propagation-relevant set. This encourages the model to associate channel structure with the scene geometry that actually shaped the observation. To capture dominant propagation effects, including blockage, reflection, and diffraction, without overfitting fine-grained geometry, the reconstruction target is limited to coarse occupancy at each masked location.

The second is \emph{physics-guided cross-modal distillation}, which shapes the \texttt{[PHYSC]} token. We train it to reconstruct a physical quantity jointly determined by the environment, the link geometry, and the observed channel: the \emph{spatial spectrum} which describes how signal energy is distributed over propagation directions and reveals the dominant angles of departure and arrival. Since those directions are determined by the propagation medium and the link geometry, anchoring \texttt{[PHYSC]} to them ties its representation to physically interpretable propagation characteristics that remain meaningful across environments. We further align \texttt{[PHYSC]} with the part of the scene that produced it: the encoded tokens of the propagation-relevant scene patches, the structures that block, reflect, or diffract this particular link, are pooled into a scene representation, and \texttt{[PHYSC]} is encouraged to match it, while representations of unrelated links are pushed apart, with the separation relaxed for spatially nearby links that plausibly interact with the same structures. This alignment is the mechanism on which the rest of the paper depends: it is what makes $\mathbf{z}_k$ a summary of the scene structure responsible for the observation. Further pretraining details are provided in Appendix~\ref{app:pretrain}.

After pretraining, we freeze $f_\phi$ and use the resulting \texttt{[PHYSC]} embedding $\mathbf{z}_k=f\phi(H_k,\mathbf{p}_k)$ as the compact representation of observation $k$, collecting $\mathcal{Z} = \{\mathbf{z}_k\}_{k=1}^{K}$ for each scene. 
These propagation-grounded representations serve as the basis for reconstructing the target environment.

\begin{figure}[t]

    \centering
    \includegraphics[width=0.9\linewidth]{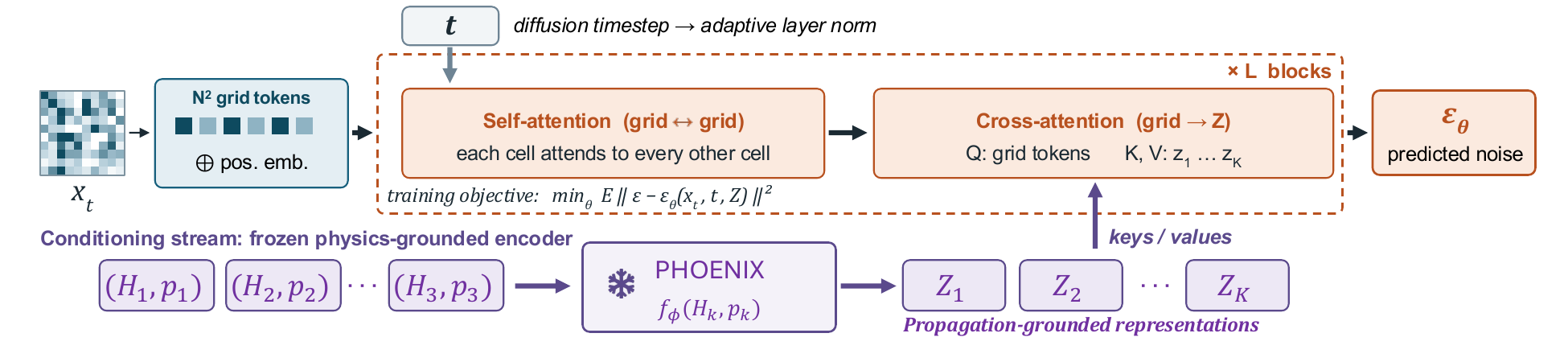}    
    \vspace{-5mm}
    \caption{
       Posterior diffusion training over physics-grounded representations from wireless foundation model allowing conditional denoising of individual cells based on observed measurements. 
    }
    \label{fig:diff}
    \vspace{-4mm}
\end{figure}

\vspace{-4mm}
\subsection{Diffusion posterior over scenes}
\label{sec:diffusion}
\vspace{-2mm}
The physics-grounded representations provide a compact description of what each wireless observation reveals about the environment. The remaining challenge is to combine these partial and potentially ambiguous views into a posterior over complete scene configurations. We therefore model $p_\theta(\mathbf{s}|\mathcal{Z})$ directly, with the second stage learning the posterior by minimizing $\mathbb{E}{(\mathbf{s},\mathcal{O})\sim\mathcal{D}}\bigl[-\log p\theta(\mathbf{s}\mid\mathcal{Z})\bigr] $.

\vspace{-3mm}
\paragraph{Parameterization.}
Although the target scene layout is expressed in a binary format, we use a continuous relaxation to enable diffusion over the occupancy grid. Specifically, we represent the scene using analog bits,
$\mathbf{x}_0 = 2\mathbf{s}-\mathbf{1}\in\{-1,+1\}^{N^2}$,
and recover the binary occupancy through thresholding,
$\hat{\mathbf{s}}=\mathbf{1}[\mathbf{x}_0>0]$.
We then define a forward diffusion process that progressively perturbs $\mathbf{x}_0$ with Gaussian noise following the standard denoising diffusion formulation~\cite{ho2020denoising}, using the cosine noise schedule $\bar{\alpha}_t$~\citep{nichol2021improved}:
\begin{equation}
q(\mathbf{x}_t\mid\mathbf{x}_0)
=
\mathcal{N}\!\left(
\mathbf{x}_t;\,
\sqrt{\bar{\alpha}_t}\mathbf{x}_0,\,
(1-\bar{\alpha}_t)\mathbf{I}
\right),
\qquad
\mathbf{x}_t
=
\sqrt{\bar{\alpha}_t}\mathbf{x}_0
+
\sqrt{1-\bar{\alpha}_t}\boldsymbol{\epsilon},
\label{eq:forwarddiff}
\end{equation}
where $\boldsymbol{\epsilon}\sim\mathcal{N}(\mathbf{0},\mathbf{I})$. This closed-form parameterization allows us to obtain a noisy scene $\mathbf{x}_t$ directly from the clean scene at any diffusion timestep.

To reverse this corruption, the denoiser $\boldsymbol{\epsilon}_\theta$ predicts the noise added to the scene while conditioning on the propagation-grounded representations $\mathcal{Z}$. It is implemented as a Transformer over $N^2$ grid tokens, one per cell, each combining the noisy occupancy value at that cell with a positional embedding of its location relative to the base station. Every block applies self-attention among the grid tokens, with the diffusion timestep injected through adaptive layer normalization with zero-initialized gates, followed by cross-attention in which the grid tokens are queries and $\mathcal{Z}$ supplies keys and values. Cross-attention suits this conditioning structure for two reasons: it is permutation invariant and accepts a variable number of observations, so one model serves any $K$ without retraining; and since each $\mathbf{z}_k$ encodes the scene structure that shaped observation $k$, cross-attention allows each grid cell to draw selectively on the observations carrying evidence about that part of the scene. Fig.~\ref{fig:diff} illustrates the overall framework, including the conditioning and denoising process.

\textbf{Training objective.}
We optimize the denoiser using the standard noise-prediction objective for diffusion models. For each training example, we sample a diffusion timestep $t$ and Gaussian noise $\boldsymbol{\epsilon}$, construct the corresponding noisy scene $\mathbf{x}_t$, and train the model to recover the injected noise:
\vspace{-2mm}
\begin{equation}
\theta^{\star}
=
\arg\min_{\theta}\;
\mathbb{E}_{(\mathbf{s},\mathcal{O})\sim\mathcal{D}}\,
\mathbb{E}_{t\sim\mathcal{U}\{1,\dots,T\},\,
\boldsymbol{\epsilon}\sim\mathcal{N}(\mathbf{0},\mathbf{I})}
\left[
\left\|
\boldsymbol{\epsilon}
-
\boldsymbol{\epsilon}_{\theta}
\left(
\mathbf{x}_t,\,
t,\,
\mathcal{Z}
\right)
\right\|_2^2
\right],
\vspace{-4mm}
\label{eq:diffloss}
\end{equation}

where $\mathbf{x}_t$ is obtained from $\mathbf{x}_0=2\mathbf{s}-\mathbf{1}$ and $\boldsymbol{\epsilon}$ through \eqref{eq:forwarddiff}, and $\mathcal{Z}=f_\phi(\mathcal{O})$ is computed by the frozen PHOENIX encoder. We update only the diffusion parameters $\theta$. Architecture details and the noise schedule are provided in Appendix~\ref{app:diffusion}.

\textbf{Readout by consensus.}
At inference, we draw $M$ independent scenes from the learned posterior by ancestral sampling, with each trajectory initialized from Gaussian noise and driven by an independent noise sequence. Each sample is then binarized to obtain an occupancy grid $\hat{\mathbf{s}}^{(m)} \sim p_\theta(\cdot\mid\mathcal{Z})$. We estimate the posterior marginal for each cell by averaging across samples and obtain the final reconstruction by thresholding this marginal:
\vspace{-1mm}
\begin{equation}
\hat{\pi}_i
=
\frac{1}{M}\sum_{m}\hat{s}_i^{(m)},
\qquad
\hat{s}_i^{\mathrm{cons}}
=
\mathbf{1}\!\left[\hat{\pi}_i>\tfrac{1}{2}\right].
\vspace{-3mm}
\label{eq:consensus}
\end{equation}

This readout follows directly from the posterior formulation. Proposition~\ref{prop:consensus} in Appendix~\ref{app:proofs} shows that $\hat{\pi}_i$ concentrates around the learned posterior marginal $\pi_i^\theta=p_\theta(s_i=1\mid\mathcal{Z})$ at a Hoeffding rate. Consequently, $\hat{s}_i^{\mathrm{cons}}$ converges to the per-cell Bayes-optimal decision under $p_\theta$, where the number of samples required for reliable recovery at a cell scales inversely with the squared posterior margin $|\pi_i^\theta-\tfrac{1}{2}|^2$. Proposition~\ref{prop:limit} shows that, when the posterior is sufficiently close to the underlying distribution, this decision recovers the true scene on the observable set $\mathcal{A}$, with additional measurements improving prediction on cells they constrain, while cells outside $\mathcal{A}$ remain governed by the conditional prior. Thus, consensus converts the learned posterior into a deterministic scene estimate.

An instructive analogy to our framework is unCLIP~\citep{ramesh2022unclip}, which also inverts a frozen representation with a diffusion decoder, conditioning on a CLIP image embedding to generate an image. Our framework follows the same high-level paradigm, but the role of sampling variability differs fundamentally. In unCLIP, the embedding specifies semantic content, while diffusion stochasticity supplies visual detail the embedding leaves undetermined, making diversity desirable. In our setting, a single physical scene underlies the observations, and variation across posterior samples reflects genuine ambiguity that the measurements leave unresolved. We therefore seek the structure consistently supported by the observations rather than diverse scenes, which is what the consensus readout in \eqref{eq:consensus} provides: it retains footprints supported across samples and suppresses structure arising from unresolved uncertainty.

\vspace{-4mm}
\section{Experiments}
\label{sec:experiment}
\vspace{-4mm}
We pretrain the foundation model with the proposed physics-guided self-supervised objective on a digital-twin corpus of 11{,}000 outdoor environments and 30 million ray-traced channel observations. The pretrained model then provides physics-grounded representations that condition the posterior diffusion model in our hierarchical pipeline, a configuration we refer to as \textbf{Diffusion + FM}.
\linebreak
To provide an interpretable alternative to the diffusion formulation and isolate the contribution of joint posterior modeling, \textbf{Marginal Prob. + FM} replaces the diffusion posterior with a deterministic per-observation decoder: each observation yields a cell-wise occupancy belief and a confidence score marking the propagation-relevant region over which that belief is meaningful. The predicted beliefs are then fused into the final reconstruction. 
To our knowledge, prior work does not directly address scene reconstruction from \textit{communication-infrastructure data} and without geometric priors, precluding a direct head-to-head comparison; we therefore use controlled baselines to isolate each component of our design.
\textbf{Diffusion + Raw Meas.} conditions the identical diffusion model on raw observations and user locations, testing whether the scene posterior can be learned without the intermediate scene-aligned representation; this is closest in spirit to~\citep{hehn2025wigatr}, which reconstructs indoor geometry from a single link's received power with the assumption of knowing floor and ceiling geometry. \textbf{Marginal Prob. + Raw Meas.} predicts occupancy over the full grid from each measurement under a binary cross-entropy objective and aggregates the per-measurement predictions, without modeling the observation-specific propagation-relevant region. Training objectives and implementation details for all models are given in Appendices~\ref{app:setup}--\ref{app:baselines}.
\linebreak
All models are evaluated on $\sim$700 environments held out from every stage of training, spanning layouts from dense urban blocks to open areas with isolated structures. Each observation consists of the full wireless channel collected by a single base station together with the user's position. We report balanced accuracy, F1, and intersection-over-union (IoU) of the reconstructed occupancy grid against ground truth. Appendix~\ref{app:setup} details the array and waveform configuration, Appendix~\ref{app:metrics} defines the metrics, and Appendices~\ref{app:position}--\ref{app:noisy_pos} study the effect of incomplete/noisy position information.

 \vspace{-3mm}
\subsection{Main Results}
\vspace{-3mm}

\begin{figure}[b]
    \vspace{-5mm}
    \centering
    \includegraphics[width=1\linewidth]{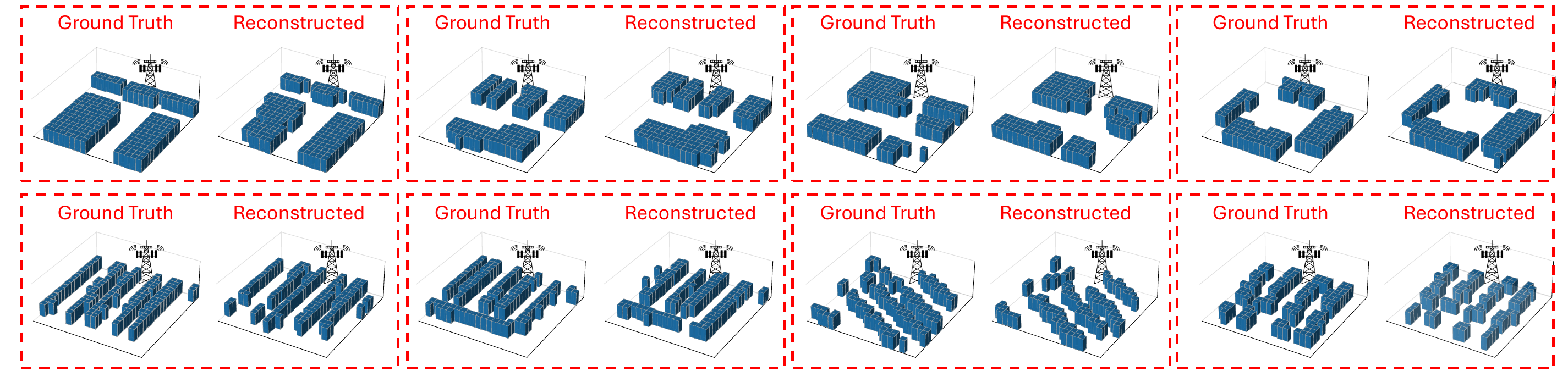}    
    \vspace{-8mm}
    \caption{\small Scene reconstruction from wireless measurements across spatially distinct environments.}
    \label{fig:results_illustration}

\end{figure}

Figure~\ref{fig:results_illustration} presents scene reconstruction produced by our framework on previously unseen representative environments. Given 500 wireless channel measurements from a single base-station node, our framework recovers the surrounding building footprints, their location, extent, and arrangement, from the wireless channel observations and user position, without any prior over scene information. The reconstructions closely follow ground-truth layouts across environments with substantially different spatial density and structure. \textit{To our knowledge, this is the first demonstration that communication-infrastructure channel measurements can support reconstruction of arbitrary outdoor scene geometry without assumptions on the number, shape, or placement of structures.} By performing posterior inference over physics-grounded representations, our framework makes this ill-posed inverse problem tractable and allows existing radio links to act as environmental sensors. The scope of what is recovered is set by the measurements: with a ULA at the base station, the channel resolves paths only in azimuth, so building height is severely underspecified; hence we reconstruct the scene as a two-dimensional occupancy grid of building footprints. A three-dimensional extension using a uniform rectangular array that can resolve paths in elevation is left for future work.

Table~\ref{tab:scene_reconstruction} summarizes the reconstruction performance. Across all metrics, our framework achieves the strongest performance, substantially outperforming both raw-measurement baselines and improving on the marginal-probability variant. To the best of our knowledge, these are the strongest scene reconstruction results reported from communication-infrastructure measurements, and substantially improve on prior attempts at this problem.
Several conclusions follow from the results.

\begin{table}[t]
\centering
\caption{Scene reconstruction performance across different models.}
\label{tab:scene_reconstruction}
\resizebox{\columnwidth}{!}{%
\begin{tabular}{lccc|ccc|ccc}
\toprule
& \multicolumn{3}{c|}{Balanced Accuracy $\uparrow$}
& \multicolumn{3}{c|}{F1 $\uparrow$}
& \multicolumn{3}{c}{IoU $\uparrow$} \\
\cmidrule(lr){2-4}
\cmidrule(lr){5-7}
\cmidrule(lr){8-10}
Model & Mean & Median & P90 & Mean & Median & P90 & Mean & Median & P90 \\
\midrule
Diffusion + Raw Meas.
& 0.505 & 0.500 & 0.519
& 0.054 & 0.016 & 0.162
& 0.030 & 0.008 & 0.088 \\

Marginal Prob. + Raw Meas.
& 0.621 & 0.618 & 0.668
& 0.397 & 0.398 & 0.503
& 0.251 & 0.248 & 0.336 \\

Marginal Prob. + FM (Ours)
& 0.738 & 0.739 & 0.803
& 0.633 & 0.642 & 0.737
& 0.470 & 0.473 & 0.584 \\

Diffusion + FM (Ours)
& \textbf{0.758} & \textbf{0.746} & \textbf{0.868}
& \textbf{0.663} & \textbf{0.655} & \textbf{0.810}
& \textbf{0.506} & \textbf{0.487} & \textbf{0.680} \\
\bottomrule
\end{tabular}%
}
\vspace{-7mm}
\end{table}

The comparison with \textit{Diffusion + Raw Meas.} isolates the benefit of the physics-grounded representation. The decoder, noise schedule, and sampler are otherwise identical, yet conditioning on raw channels collapses to near-chance performance. This indicates a failure of learnability rather than information: raw channels contain at least as much scene information as their learned representations, but the decoder must simultaneously discover the propagation-to-geometry mapping and perform posterior generation from a high-dimensional input. In contrast, the foundation model provides scene-aligned representations that substantially simplify posterior inference. 

\textit{Marginal Prob. + Raw Meas.} fails for a different reason. Its full-grid cross-entropy requires every observation to predict every scene cell, even though a measurement is informative only about propagation-relevant regions. The resulting overconfident predictions in uninformed regions accumulate during fusion, where observations that never sensed a cell still contribute votes to it. \textit{Marginal Prob. + FM} uses the same fusion but forms each per-observation belief on the physics-grounded representation and around the propagation-relevant cells of that observation, abstaining elsewhere; each vote is therefore informative where it is cast, and the aggregation better recovers the scene. Our proposed design improves over this baseline by \textbf{22.1\%} in balanced accuracy and \textbf{67.0\%} in F1 score.

Lastly, the results demonstrate that the joint posterior outperforms fused marginals in this setting. The main advantage is concentrated in the upper tail of the per-scene distribution: the median IoU of the two models differs by three percent relative, while at the 90th percentile the gap widens to sixteen percent. The marginal variant treats observations independently and combines them through a fixed rule: confidence gating discards evidence that is weak in each observation but informative in combination, and nothing in the fusion enforces that the result be a plausible scene. The diffusion decoder instead learns the joint posterior over the entire grid given the full set of observations. It fuses evidence in representation space, and it completes the scene through the learned prior where the measurements are silent, the second factor in \eqref{eq:factorization}. Additional illustrative scene reconstruction examples comparing all four methods across diverse environments are provided in Appendix~\ref{app:visual_scene_recons}.


\vspace{-4mm}
\subsection{Physical Interpretability of the Learned Beliefs}
\vspace{-3mm}

\begin{wrapfigure}{r}{0.55\textwidth}
\centering
\vspace{-4mm}
\includegraphics[width=\linewidth]{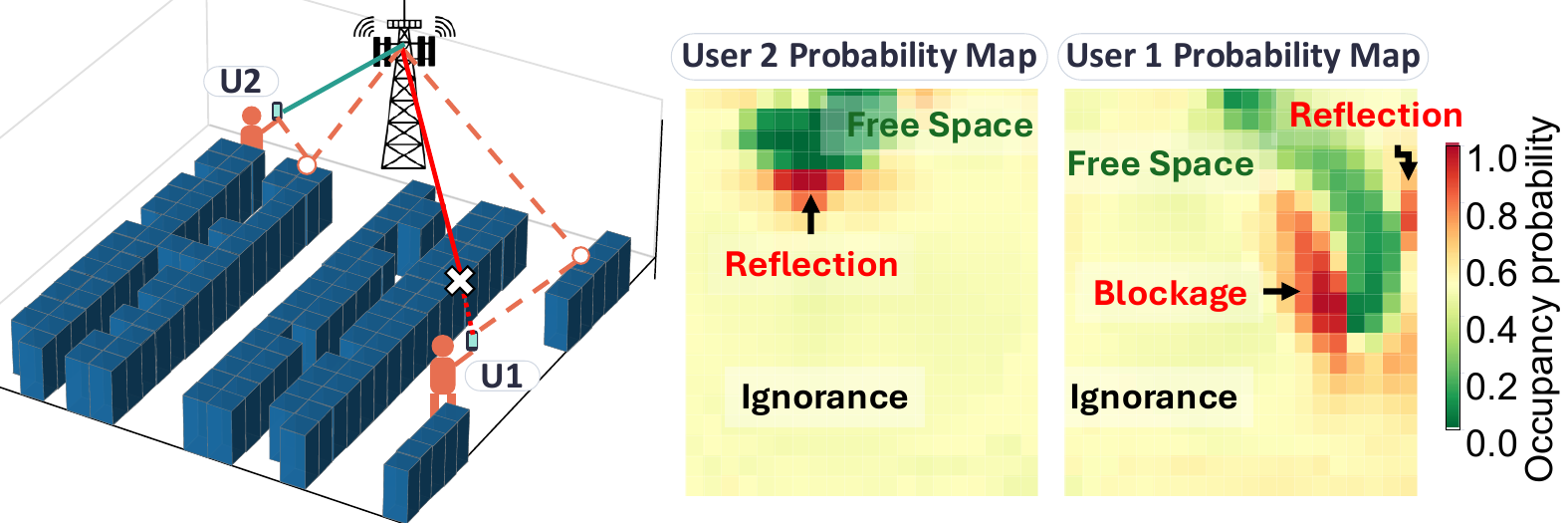}
\vspace{-8mm}
\caption{\small Marginal occupancy probability estimated from wireless measurements.}
\label{fig:interpretability}
\vspace{-4mm}
\end{wrapfigure}
We examine the physical interpretability of the per-observation marginal predictions using \textit{Marginal Prob. + FM}, which exposes a belief for each measurement individually. Specifically, we ask whether each belief captures the environment regions implicated by its propagation paths while remaining uncertain elsewhere. Figure~\ref{fig:interpretability} visualizes the marginal occupancy probability for two measurements in the same environment, together with the ground-truth layout and corresponding propagation paths. The belief recovers the free-space corridors along the direct and reflected paths, the reflecting structures, and the obstructing component when the direct path is blocked. Elsewhere, predictions remain near one-half, indicating uncertainty in regions not informed by the measurement.
\linebreak
What the beliefs recover is therefore not an arbitrary partial reconstruction but the specific geometry
that is physically responsible for shaping the observation, leaving unrelated regions unknown. This behavior, induced by the physics-guided pretraining, makes the marginal predictions physically interpretable: each belief can be read back as a statement about which structures the link encountered.

\vspace{-2mm}
\subsection{Performance Across Scene Density}
\vspace{-2mm}
Scene density governs the complexity of the propagation conditions and with it the difficulty of the inverse problem: in sparse environments most links propagate through free space, and each measurement admits a narrow family of consistent scenes, whereas blockage, reflection, and higher-order interaction widen that and force the posterior to reconcile many such families at once. We therefore partition the test environments by occupied fraction, spanning 1--70\% across test set, into three equal-width bands: sparse, moderate, and dense, and report performance within each in Table~\ref{tab:density}.

\begin{table}[h]
\centering
\vspace{-6mm}
\caption{Mean reconstruction performance by scene density.}
\label{tab:density}
\resizebox{\columnwidth}{!}{%
\begin{tabular}{lccc|ccc|ccc}
\toprule
& \multicolumn{3}{c|}{Sparse (88 scenes)}
& \multicolumn{3}{c|}{Moderate (565 scenes)}
& \multicolumn{3}{c}{Dense (18 scenes)} \\
\cmidrule(lr){2-4} \cmidrule(lr){5-7} \cmidrule(lr){8-10}
Model & Bal. Acc. $\uparrow$ & F1 $\uparrow$ & IoU $\uparrow$
      & Bal. Acc. $\uparrow$ & F1 $\uparrow$ & IoU $\uparrow$
      & Bal. Acc. $\uparrow$ & F1 $\uparrow$ & IoU $\uparrow$ \\
\midrule
Diffusion + Raw Meas.
& 0.518 & 0.066 & 0.039
& 0.503 & 0.048 & 0.027
& 0.508 & 0.179 & 0.105 \\
Marginal Prob. + Raw Meas.
& 0.647 & 0.434 & 0.283
& 0.617 & 0.393 & 0.248
& 0.589 & 0.316 & 0.190 \\
Marginal Prob. + FM (Ours)
& 0.764 & 0.631 & 0.468
& 0.735 & 0.635 & 0.471
& 0.712 & 0.607 & 0.457 \\
Diffusion + FM (Ours)
& \textbf{0.786} & \textbf{0.665} & \textbf{0.510}
& \textbf{0.752} & \textbf{0.659} & \textbf{0.501}
& \textbf{0.806} & \textbf{0.775} & \textbf{0.652} \\
\bottomrule
\end{tabular}%
}
\vspace{-2mm}
\end{table}

\textit{Diffusion + FM} is best in every regime, and its advantage is most pronounced for dense scenes. Both marginal estimators degrade monotonically in balanced accuracy as density increases, whereas \textit{Diffusion + FM} remains robust across the three bands; relative to \textit{Marginal Prob. + FM}, its balanced-accuracy advantage grows from 2.9\% to 13.2\% and its IoU advantage from 9.0\% to 42.7\%. This is consistent with how the two combine evidence: denser environments leave each observation informative about a smaller portion of the scene, and while the marginal formulation must aggregate independently estimated beliefs under this incomplete evidence, the diffusion model infers the scene jointly and completes regions where direct evidence is limited through the learned prior. 

\vspace{-4mm}
\subsection{Effect of the Number of Wireless Measurements}
\label{sec:sweep}
\vspace{-2mm}

Figure~\ref{fig:acc_vs_measurement} reports reconstruction quality for 50 to 1,000 measurements per scene, showing how
additional observations provide complementary information. Since both diffusion decoders are trained with {\small $K \in [100, 500]$}, the extreme sweep points also test generalization in the number of observations.

\begin{wrapfigure}{r}{0.43\textwidth}
\vspace{-2mm}
    \centering
    \includegraphics[width=\linewidth]{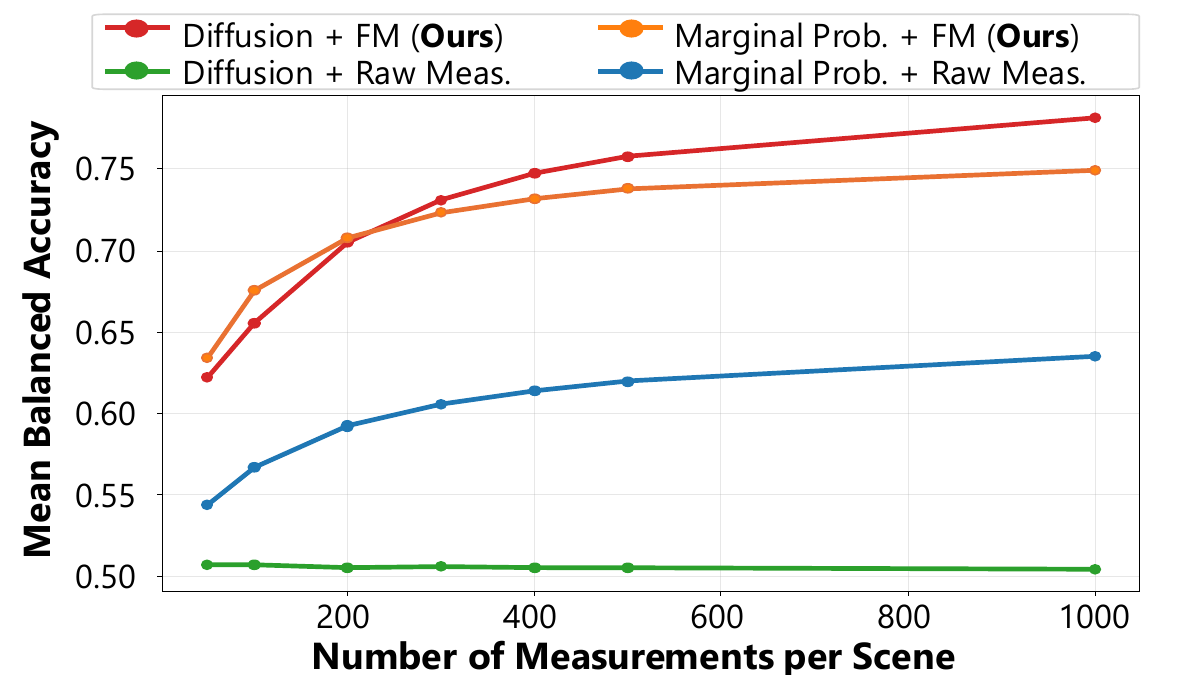}
    \vspace{-7mm}
    \caption{\small Reconstruction accuracy as a function of collected measurements per scene.}
    \label{fig:acc_vs_measurement}
    \vspace{-4mm}
\end{wrapfigure}

\textit{Diffusion + FM} and the marginal-probability variants improve with more measurements before saturating. Additional observations help in two ways: new links, particularly new reflected paths, reach cells no earlier observation touched and enlarge the observable set, while additional coverage of observed cells sharpens the posterior there, as described by Proposition~\ref{prop:limit}. Saturation follows from the same analysis, since once the users of a single base station have illuminated most of the scene its links can reach, further observations revisit covered cells and the remaining errors concentrate on structures no link reaches. \textit{Diffusion + Raw Meas.} instead stays near chance throughout: without a representation exposing the scene-relevant structure of the observations, more raw data does not become usable evidence.

The two foundation-model approaches scale differently. \textit{Marginal Prob. + FM} is comparable to \textit{Diffusion + FM} under few or moderately many observations, as both rely on the same representations under similarly incomplete evidence, but is increasingly surpassed as measurements accumulate. The value of modeling the joint posterior thus emerges once enough complementary observations are available to integrate, reflecting the complementary roles of our two components.

\vspace{-4mm}
\section{Conclusion}
\vspace{-4mm}
We presented a hierarchical framework for reconstructing environmental geometry from wireless channel measurements by combining physics-grounded representation learning with joint generative posterior inference. A wireless foundation model first extracts scene-relevant information from individual observations, while a diffusion model jointly integrates these representations to infer the posterior over scene occupancy, from which a per-cell consensus over sampled scenes recovers the building footprints. Results demonstrate our framework can recover previously unseen scene structures from communication signals, and controlled comparisons show the importance of our hierarchical design for effectively solving the underlying ill-posed inverse problem. The results highlight the potential of existing wireless infrastructure to serve for environmental perception.



\subsection*{AI use statement}

We used generative AI tools as a general-purpose writing and coding assistant. Specifically, they were used to polish the phrasing of text drafted by the authors, and to assist with the implementation of experiment and figure-generation scripts. We did not use generative AI tools for research ideation, for deriving the theoretical results, or for retrieving or summarizing related work. All AI-assisted output was reviewed by the authors: text was checked for accuracy against the underlying results, all AI-assisted code was read, executed, and verified against the reported numbers, and the theoretical claims, their assumptions, and their proofs were derived and checked by the authors. We take responsibility for the final content of this work, including all text, claims, and artifacts produced with the aid of generative AI.

\subsection*{Reproducibility statement}

We have aimed to make every component of this work reproducible. The problem formulation, the assumptions underlying each theoretical claim, and complete proofs of propositions are given in Appendices~\ref{app:illposed} and~\ref{app:proofs}. Dataset construction, including the simulation configuration, the derivation of occupancy grids, and the extraction of propagation-relevant sets, is described in Appendices~\ref{app:forward} and~\ref{app:setup}. Architecture and tokenization details for the foundation model are given in Appendix~\ref{app:arch}, and the complete pretraining objective in Appendix~\ref{app:pretrain}. The diffusion decoder, its noise schedule, and its sampler are specified in Appendix~\ref{app:diffusion}, and full training configurations --- optimizer, learning rate, schedule, batch size, and hardware --- for the proposed pipeline, the ablated variant, and both baselines are given in Appendices~\ref{app:setup}, \ref{app:marginal}, and~\ref{app:baselines}. Evaluation metrics are defined in Appendix~\ref{app:metrics}, and the evaluation protocol, including how observation subsets are drawn across the measurement sweep, is described in Appendix~\ref{app:setup}. Source code for training and evaluation, together with the pretrained dataset and checkpoints, will be provided upon publication. 



\bibliography{iclr2027_conference}
\bibliographystyle{iclr2027_conference}

\appendix
\section{Forward model details}
\label{app:forward}

\textbf{Array responses.} For a uniform linear array of $N$ elements with spacing $d$ at wavelength $\lambda$, the array response at angle $\vartheta$ is
\begin{equation} a(\vartheta) \;=\; \Bigl[\,1,\; e^{-j\frac{2\pi d}{\lambda}\sin\vartheta},\;\dots,\; e^{-j\frac{2\pi d (N-1)}{\lambda}\sin\vartheta}\,\Bigr]^{\!\top}. \label{eq:arrayresponse} \end{equation}
Substituting \eqref{eq:arrayresponse} into the path sum of the main text gives the channel entry for each transmit--receive pair and subcarrier. The angular resolution used below follows from the Rayleigh criterion for this aperture, $\Delta\theta \approx 2/N_{tx}$ radians near broadside, and the delay resolution is $\Delta\tau = 1/B$ for occupied bandwidth $B$.

\textbf{Path generation.} The path set $\mathcal{P}(\mathbf{s},\mathbf{p})$ is produced by ray tracing the scene geometry for the transmitter--receiver pair. Each path is characterized by its interaction sequence with the structures of the scene: the direct path when it exists, first- and higher-order specular reflections off building facades, and diffractions at building edges. The gain $\alpha_\ell$ aggregates free-space path loss with the reflection and diffraction coefficients along the interaction sequence.

\textbf{Propagation-relevant set.} For observation $k$ we define
\begin{equation} a_k \;=\; \bigl\{\, i \;:\; \mathrm{cell}_i \cap \ell \neq \emptyset \ \ \text{for some } \ell \in \mathcal{P}(\mathbf{s},\mathbf{p}_k) \,\bigr\}, \qquad A \;=\; \bigcup_{k=1}^{K} a_k, \label{eq:relevant} \end{equation}
recording, for every path, the grid cells traversed by its segments together with the cells at which it interacts with a structure. We additionally record the first-order neighborhood $\partial a_k$ of these cells, used to define the boundary zone in the analysis. 

\section{Ill-posedness of the inverse problem}
\label{app:illposed}

This appendix formalizes the three sources of non-uniqueness described in the main text.

\subsection{Exact invariance outside the propagation-relevant set}

By the definition \eqref{eq:relevant}, cells outside $a_k$ do not participate in any path of observation $k$, so the likelihood of that observation does not depend on them; aggregating over the observation set gives the invariance stated in the main text. Its consequence for the number of scenes consistent with a measurement is the following.

\begin{proposition}[Size of the consistent family] \label{prop:count} Fix an observation $k$ and let $\mathbf{s}^{\star}$ be the true scene. Let $\mathcal{S}_k(H_k) = \{\mathbf{s} : \mathcal{F}(\mathbf{s},\mathbf{p}_k) = H_k\}$ be the set of scenes reproducing the noiseless measurement exactly. Then $|\mathcal{S}_k(H_k)| \geq 2^{|a_k^{c}|}$. \end{proposition}

\begin{proof} By the definition of $a_k$, no cell in $a_k^{c}$ participates in any path of $\mathcal{P}(\mathbf{s}^{\star},\mathbf{p}_k)$; the path parameters $(\alpha_\ell,\tau_\ell,\theta_\ell,\phi_\ell)$ therefore do not depend on $\mathbf{s}_{a_k^{c}}$, and neither does the path sum. Hence for any subset $\mathcal{J} \subseteq a_k^{c}$ the scene $\mathbf{s}^{\star} \oplus \mathbf{e}_{\mathcal{J}}$, obtained by flipping the cells in $\mathcal{J}$, satisfies $\mathcal{F}(\mathbf{s}^{\star}\oplus \mathbf{e}_{\mathcal{J}}, \mathbf{p}_k) = H_k$ and lies in $\mathcal{S}_k(H_k)$. Distinct subsets give distinct scenes, and there are $2^{|a_k^{c}|}$ of them. \end{proof}

The inequality is strict in general, since the count enumerates only scenes differing from $\mathbf{s}^{\star}$ outside $a_k$; the ambiguities of Sections~\ref{app:resolution} and \ref{app:role} contribute further members that differ \emph{inside} $a_k$. Aggregating over the observation set replaces $a_k^{c}$ by $A^{c}$, giving $|\mathcal{S}(\mathcal{O})| \geq 2^{|A^{c}|}$: the residual non-uniqueness is exponential in the number of cells no path touches.

\subsection{Finite resolution}
\label{app:resolution}

Within $A$ the map from geometry to channel remains many-to-one. A measurement of bandwidth $B$ resolves path delays only to $\Delta\tau = 1/B$, and an aperture of $N_{tx}$ elements resolves angles only to $\Delta\theta \approx 2/N_{tx}$. If two scenes induce path parameter sets that agree within these limits,
\begin{equation} \bigl\{(\tau_\ell,\theta_\ell,\phi_\ell,\alpha_\ell)\bigr\}_{\mathbf{s}} \;\equiv\; \bigl\{(\tau_\ell,\theta_\ell,\phi_\ell,\alpha_\ell)\bigr\}_{\mathbf{s}'} \quad \text{within } (\Delta\tau, \Delta\theta), \label{eq:resolution} \end{equation}
then $\mathcal{F}(\mathbf{s},\mathbf{p}_k) \approx \mathcal{F}(\mathbf{s}',\mathbf{p}_k)$ and the two scenes cannot be told apart from observation $k$. This source of ambiguity shrinks as bandwidth and aperture grow.

\subsection{Propagation-Equivalent Scenes}
\label{app:role}

By the path-sum forward model, the channel of link $k$ depends on the scene only through the path set it induces, $\mathcal{P}(\mathbf{s},\mathbf{p}_k)$, and its parameters. Two scenes that induce the same path parameters for every link are therefore indistinguishable from $\mathcal{O}$, regardless of how they differ as occupancy grids. We call such scenes \emph{propagation-equivalent}.

\begin{proposition}[Propagation-equivalent scenes are indistinguishable] \label{prop:role} If two scenes $\mathbf{s},\mathbf{s}'$ satisfy $\mathcal{P}(\mathbf{s},\mathbf{p}_k) = \mathcal{P}(\mathbf{s}',\mathbf{p}_k)$ for a given observation, in the sense that the two path sets carry identical parameters $(\alpha_\ell,\tau_\ell,\theta_\ell,\phi_\ell)$, then $\mathcal{F}(\mathbf{s},\mathbf{p}_k)=\mathcal{F}(\mathbf{s}',\mathbf{p}_k)$, no estimator can distinguish them from $\mathcal{O}$, irrespective of bandwidth and aperture. \end{proposition}

The channel entry of link $k$ is a sum over $\ell \in \mathcal{P}(\mathbf{s},\mathbf{p}_k)$ of terms depending only on $(\alpha_\ell,\tau_\ell,\theta_\ell,\phi_\ell)$. Equality of the parameterized path sets therefore implies equality of the path sums and of the likelihoods $p(H_k\mid\cdot)$ for that observation.

The content of this statement lies in how large the set of propagation-equivalent scenes is. The map from occupancy to path parameters is many-to-one, and it is so for reasons that are geometric rather than instrumental. A cell influences a link only through the \emph{role} it plays for that link,
\begin{equation} \rho_k(i) = \rho\bigl(i,\mathbf{p}_k,\mathbf{s}\bigr) \in \{\texttt{block},\ \texttt{reflect},\ \texttt{free},\ \texttt{none}\}, \label{eq:roledef} \end{equation}
where $\texttt{block}$ denotes an occupied cell that removes a path that would otherwise reach the receiver, $\texttt{reflect}$ an occupied cell that supports a specular or diffracted interaction contributing a path, $\texttt{free}$ an unoccupied cell traversed by a contributing path, and $\texttt{none}$ a cell that neither carries nor obstructs any contributing path. Additionally, $\rho_k(i)$ is a function of the link geometry, not of $s_i$ alone: different occupancy value can provide the same role for a given link.

\emph{Aliasing of interaction sequences.} Distinct interaction sequences can produce the same path parameters. A single specular reflection from one surface and consecutive reflections from two others can yield the same total delay and the same departure and arrival angles, in which case the two scenes contribute identical terms to the path sum. Here the roles themselves differ --- different cells reflect --- yet the scenes remain propagation-equivalent. This mechanism is exact when the geometry makes the parameters coincide, as symmetric layouts often do, and approximate otherwise, in which case it compounds with the finite-resolution ambiguity of Appendix~\ref{app:resolution}.

\emph{Occupancy does not imply blockage.} A cell may be occupied and still fail to obstruct a link whose line-of-sight segment passes over it. The base station and the user sit at different heights, so the direct ray clears any structure shorter than its elevation at that point, and a two-dimensional occupancy grid records only that a building is present, not whether it is tall enough to intercept the ray. A cell can therefore belong to the propagation-relevant set of an observation --- it lies along a path the link uses --- while the measurement still carries no evidence against its being occupied: the unobstructed channel is equally consistent with free space and with a structure the ray passes above.

\subsection{Collective observations}

Under conditional independence of the measurements given the scene,
\begin{equation} p(\mathbf{s} \mid \mathcal{O}) \;\propto\; p(\mathbf{s}) \prod_{k=1}^{K} p\bigl(H_k \mid \mathbf{s}\bigr), \label{eq:collective} \end{equation}
so additional observations intersect the families $\mathcal{S}_k$ and sharpen the posterior. The intersection does not collapse to a point: by Proposition~\ref{prop:count} the residual family has size at least $2^{|A^{c}|}$, and within $A$ the ambiguities of Sections~\ref{app:resolution} and \ref{app:role} persist. Substituting the invariance into Bayes' rule gives the factorization stated in the main text, whose second factor is independent of $\mathcal{O}$; consequently a cell $i \notin A$ has $\pi_i = p(s_i = 1 \mid \mathbf{s}_{A})$, and as the data sharpen $p(\mathbf{s}_A\mid\mathcal{O})$ its variance converges to the prior-conditional variance rather than to zero.

\section{Foundation model architecture and tokenization}
\label{app:arch}

\textbf{Channel tokenization.} The CSI tensor is converted from real and imaginary components to amplitude and phase $(a_{t,r,s}, \phi_{t,r,s})$. For each transmit antenna $t$, receive antenna $r$, and group $g$ of $n_{sc}$ neighboring subcarriers, one channel token is formed by concatenating the corresponding amplitude and phase values, $\mathbf{x}^{ch}_{t,r,g} = [\,a_{t,r,g,1:n_{sc}} \,\|\, \phi_{t,r,g,1:n_{sc}}\,] \in \mathbb{R}^{2n_{sc}}$, giving $N_{ch} = N_{tx}N_{rx}N_{sc}/n_{sc}$ tokens. Each is mapped to the shared embedding dimension $d$ by a learnable linear projection followed by layer normalization. This preserves full transmit- and receive-antenna resolution, since each antenna pair carries distinct spatial information, while grouping subcarriers to represent local frequency structure compactly.

\textbf{Scene and location tokenization.} The height map $\mathbf{S} \in \mathbb{R}^{N\times N}$ is partitioned into non-overlapping $p\times p$ patches, flattened, and projected analogously, yielding $(N/p)^2$ scene tokens. The receiver location is projected by a single linear layer into one location token.

\textbf{Physically grounded positional and modality embeddings.} Tokenization captures content but not physical origin. For channel tokens we encode position separately along the transmit, receive, and subcarrier axes using sinusoidal encodings and concatenate them, $\mathbf{e}^{ch}_{t,r,g} = [\,\mathrm{PE}_{tx}(t) \,\|\, \mathrm{PE}_{rx}(r) \,\|\, \mathrm{PE}_{sc}(g)\,]$, preserving both the spatial structure of the arrays and the frequency structure of the channel. For scene patches we apply a two-layer MLP to sinusoidal features of the patch's normalized distance and bearing relative to the base station, so each token carries its geometric relationship to the reference position. A learnable modality embedding $\mathbf{m}^{(i)} \in \{\mathbf{m}^{ch}, \mathbf{m}^{scn}, \mathbf{m}^{loc}\}$ identifies the token type. The final input token is the sum of the three: what the token contains, where it comes from physically, and which modality it belongs to.

\textbf{Encoder.} Because channel tokens vastly outnumber scene and location tokens, lightweight pre-encoder blocks first operate on channel tokens alone, using rotary position embeddings across the $(t,r,g)$ axes to capture local correlations across nearby antennas and frequencies before cross-modal fusion; absolute encodings are retained for the subsequent stage. All tokens are then concatenated into a single sequence prefixed by \texttt{[PHYSC]} and processed by a single-stream transformer with merged-token self-attention, so all modalities interact from the first layer. A binary attention mask enforces the asymmetry described in the main text: \texttt{[PHYSC]} is an unrestricted query over all tokens but is excluded as a key everywhere, so it can read from the sequence without any token reading from it.

\section{Pretraining objective in full}
\label{app:pretrain}

\textbf{Masked channel reconstruction.} With $\mathcal{M}_{ch}$ the masked channel tokens,
\begin{equation} \mathcal{L}_{ch} \;=\; \sqrt{\frac{1}{|\mathcal{M}_{ch}|}\sum_{i \in \mathcal{M}_{ch}} \bigl\|\hat{\mathbf{h}}_i - \mathbf{h}_i\bigr\|_2^2}. \label{eq:lch} \end{equation}

\textbf{Masked scene reconstruction.} Masking is restricted to the propagation-relevant patches of the link. With $\mathcal{M}_{scn} \subseteq a_k$ the masked propagation-relevant patches and $o_i$ the occupancy label,
\begin{equation} \mathcal{L}_{scn} \;=\; -\frac{1}{|\mathcal{M}_{scn}|}\sum_{i \in \mathcal{M}_{scn}} \bigl[\, o_i \log \hat{o}_i + (1-o_i)\log(1-\hat{o}_i) \,\bigr]. \label{eq:lscn} \end{equation}
Restricting the mask to $a_k$ is what prevents the model from being trained to infer regions that left no trace in the channel.

\textbf{Masked location reconstruction.} $\mathcal{L}_{loc}$ is the root-mean-square error between predicted and true receiver location.

\textbf{Spatial-spectrum distillation.} Let $H^{(\Psi)}_n$ denote slices of the CSI tensor along the transmit or receive antenna axis at a reference antenna and subcarrier. The angle-of-departure and angle-of-arrival spectra are
\begin{equation} s_{\Psi}[b] \;=\; \Bigl|\sum_{n=0}^{N_\Psi - 1} H^{(\Psi)}_n\, e^{-j2\pi n b / N_\Psi}\Bigr|^2, \qquad \Psi \in \{tx, rx\}, \label{eq:spectrum} \end{equation}
with $b$ indexing angular bins. The \texttt{[PHYSC]} token is trained to reconstruct these spectra under
\begin{equation} \mathcal{L}_{\Psi} \;=\; D_{\mathrm{KL}}\bigl(\mathrm{softmax}(s_\Psi) \,\|\, \mathrm{softmax}(\hat{s}_\Psi)\bigr) \;+\; \lambda\,\mathrm{MSE}(\hat{s}_\Psi, s_\Psi), \label{eq:lspec} \end{equation}
where the KL term emphasizes alignment of spectral peak locations and the MSE term preserves their relative strengths.

\textbf{PHYSC--scene contrastive alignment.} The encoded tokens of the propagation-relevant patches are pooled into a scene representation $z^{scn}$, and
\begin{equation} \mathcal{L}_{con} \;=\; -\frac{1}{N}\sum_{i=1}^{N} \log \frac{\exp\bigl(z^{physc}_i \cdot z^{scn}_i / \tau\bigr)}{\sum_{j=1}^{N} w_{ij}\exp\bigl(z^{physc}_i \cdot z^{scn}_j / \tau\bigr)}, \label{eq:lcon} \end{equation}
with temperature $\tau$ and weights $w_{ij} \in (0,1]$ that down-weight negative pairs drawn from spatially nearby links, preventing the objective from penalizing genuinely similar propagation geometries as hard negatives.

\textbf{Pretraining objective.} The five terms are combined as
\begin{equation} \mathcal{L}_{\mathrm{pre}} \;=\; \lambda_{ch}\,\mathcal{L}_{ch} \;+\; \lambda_{scn}\,\mathcal{L}_{scn} \;+\; \lambda_{loc}\,\mathcal{L}_{loc} \;+\; \lambda_{spec}\bigl(\mathcal{L}_{\Psi}^{tx} + \mathcal{L}_{\Psi}^{rx}\bigr) \;+\; \lambda_{con}\,\mathcal{L}_{con}, \label{eq:pretrain} \end{equation}
where $\mathcal{L}_{ch}$ and $\mathcal{L}_{loc}$ are reconstruction errors on the masked channel entries and the masked location, $\mathcal{L}_{scn}$ is a binary cross-entropy on the masked propagation-relevant scene patches, $\mathcal{L}_{\Psi}$ combines a Kullback--Leibler term that aligns the peaks of the predicted and true spatial spectra with a mean-squared term that preserves their relative strengths, and $\mathcal{L}_{con}$ is the distance-weighted contrastive alignment between \texttt{[PHYSC]} and the propagation-relevant scene representation. 

\textbf{Decoder and configuration.} Following masked autoencoding, a lightweight transformer decoder receives the encoded visible tokens together with learnable tokens standing for the missing inputs and predicts the corresponding channel values, scene occupancy, and location. It is used only during pretraining and discarded afterwards. 

\section{Diffusion decoder details}
\label{app:diffusion}

\textbf{Noise schedule.} We use the cosine schedule, $\bar{\alpha}_t = f(t)/f(0)$ with $f(t) = \cos^2\!\bigl(\frac{t/T + \varsigma}{1+\varsigma}\cdot\frac{\pi}{2}\bigr)$ and offset $\varsigma = 0.008$, and $\beta_t = 1 - \bar{\alpha}_t/\bar{\alpha}_{t-1}$ clipped at $0.999$ for numerical stability.

\textbf{Denoiser.} Each block comprises timestep-modulated self-attention over the grid tokens, plain cross-attention into the conditioning set, and a timestep-modulated feed-forward network. Modulation follows the adaptive layer normalization scheme in which the timestep embedding produces shift, scale, and gate parameters per block, with gates initialized to zero so that each block begins as an identity map. Cross-attention is left unmodulated, since a variable-size conditioning set does not fold naturally into a single global modulation vector the way a scalar timestep does. A linear adapter maps the encoder embedding dimension to the decoder width, and the output head is zero-initialized.

\textbf{Conditioning dropout.} During training the entire conditioning set is replaced by a learned null embedding with probability $p_{\mathrm{drop}} = 0.15$. We use this as a regularizer against over-reliance on any particular conditioning configuration; at sampling time the conditional prediction is used directly, without guidance.

\textbf{Sampler.} We use ancestral sampling from $\mathbf{x}_T \sim \mathcal{N}(0,\mathbf{I})$. At each step the clean-scene estimate $\hat{\mathbf{x}}_0 = (\mathbf{x}_t - \sqrt{1-\bar{\alpha}_t}\,\boldsymbol{\epsilon}_\theta(\mathbf{x}_t, t, \mathcal{Z}))/\sqrt{\bar{\alpha}_t}$ is clamped to $[-1,1]$ before forming the posterior mean, which stabilizes the trajectory for a target that is binary at $t=0$. Each of the $M$ trajectories is driven by independent noise and binarized by sign at the final step, which is the condition under which the concentration results of Appendix~\ref{app:proofs} apply.

\section{Proofs of the consensus propositions}
\label{app:proofs}

In the following, we establish the propositions that characterize the theoretical properties of our consensus estimator. We first show that consensus under the learned posterior yields the Bayes-optimal per-cell decision. We then show that consensus from a near-optimal posterior recovers the true scene values on the observable set, with the reconstruction error decreasing as the number of observations increases.

\begin{proposition}[Consensus is the plug-in Bayes decision] \label{prop:consensus} Let $\hat{\mathbf{s}}^{(1)},\dots,\hat{\mathbf{s}}^{(M)}$ be i.i.d.\ draws from $p_\theta(\cdot\mid\mathcal{Z})$ and let $\pi_i^{\theta} = p_\theta(s_i=1\mid\mathcal{Z})$. Then for every cell $i$: (a) $\hat{\pi}_i \to \pi_i^{\theta}$, with
\linebreak
$\Pr[\,|\hat{\pi}_i - \pi_i^{\theta}| \geq \varepsilon\,] \leq 2e^{-2M\varepsilon^{2}}$; (b) if $\pi_i^{\theta}\neq\tfrac12$, then $\hat{s}_i^{\mathrm{cons}}$ equals the per-cell Bayes-optimal decision under $p_\theta$ with probability at least $1 - 2e^{-2M(\pi_i^{\theta}-1/2)^{2}}$; and (c) $\hat{\pi}_i(1-\hat{\pi}_i) \to \mathrm{Var}_\theta[s_i\mid\mathcal{Z}]$. \end{proposition}

\begin{proposition}[Limit under the true posterior] \label{prop:limit} Assume the model is well specified with $\mathcal{Z}$ sufficient and the observations conditionally independent given $\mathbf{s}$. For $i \notin \mathcal{A}$ the marginal equals the conditional prior $p(s_i=1\mid\mathbf{s}_{\mathcal{A}})$ and the data cannot reduce its variance. For $i \in \mathcal{A}$, let $n_i$ count the observations separating $\mathbf{s}^{\star}$ from the alternatives differing at cell $i$; the consensus recovers $s_i^{\star}$ with error decaying in $n_i$. \end{proposition}

\subsection{Proof of Proposition~\ref{prop:consensus}}

Fix a cell $i$. Each sample $\hat{s}_i^{(m)}$ is a Bernoulli variable with parameter $\pi_i^{\theta} = p_\theta(s_i = 1 \mid \mathcal{Z})$, and the samples are independent because each reverse trajectory is driven by independent noise.

\paragraph{(a) Concentration.} The variables are bounded in $[0,1]$, so Hoeffding's inequality gives, for any $\varepsilon > 0$,
\begin{equation} \Pr\Bigl[\,\bigl|\hat{\pi}_i - \pi_i^{\theta}\bigr| \geq \varepsilon \,\Bigr] \;\leq\; 2\exp\bigl(-2M\varepsilon^{2}\bigr), \label{eq:hoeffding} \end{equation}
and $\hat{\pi}_i \to \pi_i^{\theta}$ as $M\to\infty$ by the strong law of large numbers.

\paragraph{(b) Decision.} The consensus $\hat{s}_i^{\mathrm{cons}} = \mathbf{1}[\hat{\pi}_i > 1/2]$ differs from $\mathbf{1}[\pi_i^{\theta} > 1/2]$ only if $\hat{\pi}_i$ falls on the opposite side of $1/2$, which requires $|\hat{\pi}_i - \pi_i^{\theta}| \geq |\pi_i^{\theta} - 1/2|$. Applying \eqref{eq:hoeffding} with $\varepsilon = |\pi_i^{\theta} - 1/2|$,
\begin{equation} \Pr\Bigl[\hat{s}_i^{\mathrm{cons}} \neq \mathbf{1}[\pi_i^{\theta} > \tfrac12]\Bigr] \;\leq\; 2\exp\Bigl(-2M\bigl(\pi_i^{\theta}-\tfrac12\bigr)^{2}\Bigr). \label{eq:margin} \end{equation}
The limit $\mathbf{1}[\pi_i^{\theta} > 1/2]$ is the Bayes decision for cell $i$ under $p_\theta$ with $0$--$1$ loss, i.e.\ the minimizer of the expected per-cell Hamming error, since for a Bernoulli posterior the optimal decision is the more probable value.

\paragraph{(c) Variance.} The map $x \mapsto x(1-x)$ is continuous, so $\hat{\pi}_i(1-\hat{\pi}_i) \to \pi_i^{\theta}(1-\pi_i^{\theta}) = \mathrm{Var}_\theta[s_i\mid\mathcal{Z}]$ by the continuous mapping theorem. \hfill$\square$

\medskip
\noindent Two remarks. First, \eqref{eq:margin} shows that the number of samples required scales as $1/(\pi_i^{\theta}-1/2)^{2}$: cells whose marginal is close to $1/2$ are exactly those that need many draws to stabilize, and they are also the cells on which the samples disagree most. The instability of the estimate and the ambiguity it reflects are the same phenomenon. Second, the proposition concerns the \emph{learned} posterior; it says that sampling extracts a decision already encoded in $p_\theta$ and adds no information of its own.

\subsection{Proof of Proposition~\ref{prop:limit}}

Rather than assuming a uniform information gap at every covered cell, we define the relevant count directly in terms of which observations can distinguish the competing hypotheses.

\begin{definition}[Separating observation]
\label{def:separating}
Observation $k$ \emph{separates} two scenes $\mathbf{s},\mathbf{s}'$ if $\mathcal{F}(\mathbf{s},\mathbf{p}_k) \neq \mathcal{F}(\mathbf{s}',\mathbf{p}_k)$, i.e.\ the two scenes induce different path parameters for link $k$. For a cell $i$, we write $n_i$ for the number of observations that separate $\mathbf{s}^{\star}$ from the alternatives differing at cell $i$.
\end{definition}

Two cases give $X_k = 0$ and are excluded from $n_i$ by construction: observations with $i \notin a_k$, which are invariant to cell $i$; and observations that are role-degenerate for the pair, which by Proposition~\ref{prop:role} induce identical path parameters under both hypotheses despite covering the cell. Neither requires a separate assumption, and $n_i \leq |\{k : i \in a_k\}|$ in general.

We assume:
\begin{itemize}[leftmargin=3em, itemsep=1pt, topsep=2pt]
\item[(A1)] \emph{Well specification.} $p_\theta(\cdot\mid\mathcal{Z}) = p(\cdot\mid\mathcal{O})$, with $\mathcal{Z}$ sufficient for $\mathbf{s}$.
\item[(A2)] \emph{Forward invariance.} $p(\mathcal{O}\mid\mathbf{s})$ depends on $\mathbf{s}$ only through $\mathbf{s}_{\mathcal{A}}$.
\item[(A3)] \emph{Conditional independence.} The observations are independent given $\mathbf{s}$.
\item[(A4)] \emph{Regularity.} The log-likelihood ratios $X_k$ are sub-exponential, so their moment generating functions exist near the origin.
\end{itemize}

\textbf{(i) Unobservable cells.} Let $i \notin \mathcal{A}$. Factorizing the prior as $p(\mathbf{s}) = p(\mathbf{s}_{\mathcal{A}})\,p(\mathbf{s}_{\mathcal{A}^{c}}\mid\mathbf{s}_{\mathcal{A}})$ and applying (A2) to Bayes' rule gives the factorization of \eqref{eq:factorization}, in which $\mathcal{O}$ does not appear in the second factor. Marginalizing, $\pi_i = p(s_i = 1 \mid \mathbf{s}_{\mathcal{A}})$, so the belief at cell $i$ is supplied entirely by the prior conditioned on the observable part of the scene. As the data sharpen $p(\mathbf{s}_{\mathcal{A}}\mid\mathcal{O})$, the variance at cell $i$ converges to the prior-conditional variance $\mathrm{Var}_{\mathrm{prior}}[s_i\mid\mathbf{s}_{\mathcal{A}}^{\star}]$, a floor the measurements cannot lower. This holds in both regimes below.

\textbf{(ii) Observable cells.} The forward map is deterministic in the \emph{physical} scene --- the mesh geometry, building heights, and material properties from which the ray tracer computes propagation --- but the occupancy grid $\mathbf{s}$ is a coarse projection of that description. Writing $\mathbf{g}$ for the physical scene, many $\mathbf{g}$ are consistent with the same grid, and they produce different channels; the likelihood entering the posterior is therefore the marginal
\begin{equation} p(H_k \mid \mathbf{s}) \;=\; \int p\bigl(H_k \mid \mathbf{g}\bigr)\, p(\mathbf{g} \mid \mathbf{s})\, d\mathbf{g}, \label{eq:marginal_likelihood} \end{equation}
which is a density over channel realizations rather than a point mass, even when the simulator itself is noiseless. The Kullback--Leibler divergences below are accordingly well defined.

Let $\delta_i > 0$ denote the smallest divergence $\mathrm{KL}\bigl(p(H_k\mid\mathbf{s}^{\star}) \,\|\, p(H_k\mid\mathbf{s}')\bigr)$ over the separating observations counted by $n_i$ and the alternatives $\mathbf{s}'$ in contention at cell $i$; separation guarantees $\delta_i > 0$. Using (A3), the log posterior odds between $\mathbf{s}^{\star}$ and such an alternative decompose as
\begin{equation} L \;=\; \log\frac{p(\mathbf{s}^{\star}\mid\mathcal{O})}{p(\mathbf{s}'\mid\mathcal{O})} \;=\; \underbrace{\log\frac{p(\mathbf{s}^{\star})}{p(\mathbf{s}')}}_{=:L_0} \;+\; \sum_{k=1}^{K} X_k, \qquad X_k \;=\; \log\frac{p(H_k\mid\mathbf{s}^{\star})}{p(H_k\mid\mathbf{s}')}, \label{eq:logodds} \end{equation}
where $X_k \equiv 0$ for every non-separating observation, leaving $n_i$ nonzero terms. Taking expectations under the truth, $\mathbb{E}[X_k] \geq \delta_i$, so the drift $\mathbb{E}[\sum_k X_k] \geq n_i\delta_i$ grows linearly in $n_i$ while the standard deviation of the sum grows as $O(\sqrt{n_i})$ by independence. An error at cell $i$ requires $L \leq 0$, so the sum must fall below its mean by at least $n_i\delta_i + L_0$. A Chernoff bound gives, for suitable $\lambda>0$,
\begin{equation} \Pr\bigl[\textstyle\sum_k X_k + L_0 \leq 0\bigr] \;\leq\; e^{\lambda L_0}\!\!\prod_{k \ \mathrm{separating}}\!\! \mathbb{E}\bigl[e^{-\lambda X_k}\bigr] \;\leq\; e^{-c\, n_i \delta_i}, \label{eq:chernoff} \end{equation}
for a constant $c>0$ depending on the sub-exponential parameters, with $L_0$ absorbed for $n_i$ large enough. Combining with (A1) and Proposition~\ref{prop:consensus}, the consensus recovers $s_i^{\star}$ with error decaying in $n_i\delta_i$. \hfill$\square$

\textbf{Remarks.} The stochasticity in \eqref{eq:marginal_likelihood} is intrinsic to the problem rather than an artifact of measurement noise: the occupancy grid does not determine building heights, facade materials, or sub-cell geometry, all of which shape the channel, so a cell's occupancy constrains the observed channel only in distribution. This is the same representational gap discussed in Appendix~\ref{app:role}, where a link may clear an occupied cell entirely. It also bounds what any amount of data can achieve, since $\delta_i$ is limited by how much of the physical scene the grid leaves unspecified, and is therefore one reason accuracy saturates below unity in Section~\ref{sec:sweep} even as observations accumulate. The quantity governing recovery at a cell is $n_i$, the number of observations that actually separate the competing hypotheses there, which excludes both uncovered and role-degenerate observations. Role degeneracy therefore does not halt recovery; it reduces the effective number of informative observations, and only if every link in the deployment were non-separating for a pair differing at a cell would that cell remain unresolved. Regarding (A1), diffusion models are consistent in the limit of data and capacity, not at finite scale, so the proposition describes the ideal posterior and the empirical question is how closely $p_\theta$ approaches it; the comparison against a decoder conditioned on raw channels probes the sufficiency half of this assumption. Regarding (A2), real propagation leaks weak energy off the dominant paths, so in practice the boundary of $\mathcal{A}$ is soft and concentration on $\mathcal{A}^{c}$ is slow rather than exactly absent.

\vspace{-4mm}
\section{Dataset and training details}
\label{app:setup}
\vspace{-2mm}
\paragraph{Dataset.} Channels are simulated with Sionna RT~\citep{hoydis2023sionna} over more than 11{,}000 outdoor urban environments spanning 200 cities, each covering a $200\,\mathrm{m}\times 200\,\mathrm{m}$ area, with three-dimensional building geometry extracted from OpenStreetMap~\citep{openstreetmap} using the Geo2SigMap pipeline~\citep{li2024geo2sigmap}. Simulations use a carrier frequency of 5.8~GHz with 100~MHz bandwidth and 128 subcarriers. The base station is equipped with a uniform linear array of $N_{tx}=32$ antennas at a height of 15~m and serves more than 2{,}000 randomly placed users per environment, each equipped with an $N_{rx}=8$-element uniform linear array at 1.5~m. The dataset comprises over 30 million samples, each consisting of a complex channel tensor, the user location, and the building height map. Occupancy grids are obtained on a $20\times 20$ grid, i.e.\ $\Delta R = 10$~m, whose cells coincide with the scene patches of the foundation model. Propagation-relevant sets are recorded from the ray tracer as described in Appendix~\ref{app:forward}.

\textbf{Foundation model pretraining.} The encoder consists of 12 transformer layers with 512-dimensional embeddings and a 2{,}450-token context window, totaling approximately 60M trainable parameters. Pretraining uses the objective of \eqref{eq:pretrain} with progressive masking combinatorial ratios. It was performed on 32 NVIDIA L40 GPUs for 9{,}600 GPU-hours over 48 epochs, processing approximately 882 billion tokens.

\textbf{Diffusion decoder training.} The denoiser has depth 6, width 192, 6 attention heads, and MLP ratio 4, with a linear adapter mapping the 512-dimensional encoder embedding to the decoder width. Because the encoder is frozen, its \texttt{[PHYSC]} embeddings are extracted once for up to 2{,}000 observations per scene and cached, so decoder training involves no forward or backward pass through the encoder. We use $T=1000$ steps with the cosine schedule of Appendix~\ref{app:diffusion} and conditioning dropout $p_{\mathrm{drop}}=0.15$. To serve any number of observations with one model, the number of observations per training step is redrawn uniformly from $[K_{\min}, K_{\max}] = [100, 500]$ and shared across the scenes of the batch. Training uses AdamW with learning rate $10^{-4}$ and weight decay $10^{-4}$, a linear warmup over the first epoch from $0.1\times$ the base rate followed by cosine decay, and gradient-norm clipping at $1.0$, with 32 scenes per GPU on 8 GPUs for 100 epochs of 500 steps each. Samples are drawn with the conditional prediction, without guidance, and the consensus is formed over $M=32$ samples.

\textbf{Evaluation.} All models are evaluated on 671 held-out test environments. For each scene, observations are loaded once at the largest sweep value, $K=1000$, and every smaller $K$ uses a prefix of the same sequence, so observation sets are nested across sweep points and differences across $K$ reflect added observations rather than resampling. The per-observation encodings of every model are computed once per scene and reused across sweep points; only the fusion step, or the reverse-diffusion sampling for the two diffusion models, is repeated for each $K$.

\vspace{-4mm}
\section{Marginal Prob. + FM}
\label{app:marginal}
\vspace{-2mm}
\paragraph{Per-observation decoder.} This variant reuses the pretrained model in full, including the lightweight decoder of the pretraining stage. Each observation is encoded with its height-map input entirely masked, and the decoder predicts the occupancy of every scene patch from the channel and location alone,
\begin{equation} q_{k,i} \;=\; \sigma\bigl(g_\psi(f_\phi(H_k, \mathbf{p}_k))_i\bigr), \label{eq:marginalbelief} \end{equation}
which is the occupancy belief supported by observation $k$ alone. Unlike Diffusion + FM, the encoder is not frozen: $\phi$ and $\psi$ are fine-tuned jointly under the objective below.

\textbf{Confidence-shaping objective.} Let $a_k$ be the propagation-relevant set of observation $k$, $\partial a_k$ the cells within one step of $a_k$ (8-connected) that are not in $a_k$, and $u_k = (a_k \cup \partial a_k)^{c}$ the remaining cells. Write $\mathrm{BCE}_\beta(q,s) = -[\,\beta\, s\log q + (1-s)\log(1-q)\,]$ for class-balanced cross-entropy with $\beta = (1-\rho)/\rho$, where $\rho$ is the fraction of occupied cells in the scene, clipped to $[0.01, 0.99]$. The model is trained with

\begin{equation}
\begin{aligned}
\mathcal{L}_{\mathrm{marg}}
&= \frac{1}{|a_k|}\sum_{i\in a_k}\mathrm{BCE}_\beta(q_{k,i}, s_i)
-\lambda_c\,\frac{1}{|a_k|}\sum_{i\in a_k}\bigl|2q_{k,i}-1\bigr| \\
&\quad + \lambda_n\,\frac{1}{|\partial a_k|}\sum_{i\in\partial a_k}\mathrm{BCE}(q_{k,i}, s_i)
+\lambda_u\,\frac{1}{|u_k|}\sum_{i\in u_k}\Bigl(q_{k,i}-\tfrac12\Bigr)^{2}.
\end{aligned}
\label{eq:marginalloss}
\end{equation}

The first term fits the occupancy of the cells the observation actually constrains, and the second rewards confident predictions there. The third supervises the boundary of the relevant set with a reduced weight, since cells adjacent to a path are partially constrained. The fourth drives predictions on cells the observation does not reach toward $1/2$. This last term is what gives the belief its meaning: under cross-entropy alone, every cell would be pushed toward a confident label and the belief could no longer indicate where the observation is uninformative. The propagation-relevant sets are used only to form this loss during training and are not required at inference. We use $\lambda_c=0.1$, $\lambda_n=0.3$, and $\lambda_u=0.5$, with unit weight on the first term. Training uses AdamW with learning rate $10^{-5}$ and weight decay $10^{-4}$, 32 observations per GPU, a linear warmup over the first epoch followed by cosine decay, and gradient-norm clipping at $1.0$, for 20 epochs, each covering one quarter of the training corpus.

\textbf{Confidence-gated geometric fusion.} Per-observation beliefs are combined by a confidence-gated, geometrically weighted average: an observation contributes to a cell only if its belief there is confident, and in proportion to how close the cell lies to its link,
\begin{equation}
\begin{aligned}
w_{k,i} &= \gamma_{k,i}\,\mathbf{1}\bigl[\,|2q_{k,i}-1| > \tau_c\,\bigr],
\qquad
\gamma_{k,i} = \exp\!\Bigl(-\frac{d_{k,i}^{2}}{2\sigma^{2}}\Bigr), \\
P_i &= \frac{\sum_{k} w_{k,i}\,q_{k,i}}{\sum_{k} w_{k,i}},
\qquad
\hat{s}_i = \mathbf{1}\Bigl[\,P_i > \tfrac12 \ \wedge\ \sum_{k} w_{k,i} > \kappa\,\bigr].
\end{aligned}
\label{eq:fusion}
\end{equation}

where $d_{k,i}$ is the distance from the center of cell $i$ to the segment joining the base station and user $k$. The geometric weight depends only on the base-station and user positions, which are known at inference; the propagation-relevant sets are not used. Cells that receive no sufficiently weighted vote are assigned the free label, so every cell receives a decision and all metrics are computed over the full grid. We use $\tau_c = 0.75$, $\sigma = 50$~m, and $\kappa = 0.05$. The confidence gate is what allows the fusion to discard observations that are uninformative about a cell, and it is effective only because \eqref{eq:marginalloss} drives such beliefs toward $1/2$.

\section{Baselines}
\label{app:baselines}

\textbf{Diffusion + Raw Meas.} The denoiser architecture, noise schedule, sampler, and number of consensus samples are identical to those of Diffusion + FM. The conditioning set is formed from the raw observations by a learnable tokenizer $\tau_\omega$, trained jointly with the decoder,

\begin{equation} \mathcal{Z}_{\mathrm{raw}} \;=\; \bigl\{\tau_\omega(H_k, \mathbf{p}_k)\bigr\}_{k=1}^{K}, \; (\theta^{\star},\omega^{\star}) = \arg\min_{\theta,\omega}\; \mathbb{E}_{(\mathbf{s},\mathcal{O})\sim\mathcal{D}}\ \mathbb{E}_{t,\boldsymbol{\epsilon}} \Bigl[\, \bigl\|\, \boldsymbol{\epsilon} - \boldsymbol{\epsilon}_\theta\bigl(\mathbf{x}_t,\ t,\ \mathcal{Z}_{\mathrm{raw}}\bigr) \bigr\|_2^2 \,\Bigr]. \label{eq:rawdiff} \end{equation}
The tokenizer $\tau_\omega$ is a shallow transformer encoder with 3 layers, width 256, and 8 heads. Each raw observation is tokenized as in the channel front end of the foundation model --- one token per transmit--receive antenna pair and group of 16 subcarriers, giving 2{,}048 channel tokens with factored sinusoidal positional encodings along the transmit, receive, and subcarrier axes --- together with a location token, and the encoded sequence is mean-pooled into a single token per observation. The conditioning set $\mathcal{Z}_{\mathrm{raw}}$ therefore has the same cardinality $K$ as the set of \texttt{[PHYSC]} embeddings it replaces and enters the identical denoiser through a linear adapter to the decoder width. $\tau_\omega$ and $\theta$ are trained jointly from scratch with learning rate $10^{-4}$, conditioning dropout $0.15$, and $K$ drawn from $[100, 500]$, for 100 epochs over the training scenes on 32 GPUs. Because the decoder is held fixed, any difference from Diffusion + FM is attributable to the conditioning.

\textbf{Marginal Prob. + Raw Meas.} An encoder $h_\omega$ with 12 layers, width 512, and 16 heads --- capacity-matched to the foundation-model encoder and using the same channel and location tokenization --- is trained from scratch on individual raw observations, with $N^2$ learned positional query tokens, one per cell, reading out the occupancy of the full grid:
\begin{equation} \tilde{q}_{k,i} = \sigma\bigl(h_\omega(H_k, \mathbf{p}_k)_i\bigr), \; \mathcal{L}_{\mathrm{naive}} \;=\; \frac{1}{N^2}\sum_{i=1}^{N^2} w_{k,i}\,\mathrm{BCE}_\beta\bigl(\tilde{q}_{k,i}, s_i\bigr), \; w_{k,i} = \begin{cases} 3, & i \in a_k,\\ 1, & \text{otherwise}. \end{cases} \label{eq:naiveloss} \end{equation}
Training uses AdamW with learning rate $5\times10^{-5}$ and weight decay $10^{-4}$, 32 observations per GPU on 32 GPUs, a linear warmup over the first epoch followed by cosine decay, and gradient-norm clipping at $1.0$, for 32 epochs, each covering one quarter of the training corpus. The per-observation estimates are fused with the same rule as Marginal Prob. + FM, \eqref{eq:fusion}, with identical parameters, so the two models differ in representation and training objective but not in fusion. Here the confidence gate is ineffective. The loss \eqref{eq:naiveloss} supervises every cell with a hard label whether or not the observation constrains it: relevant cells are up-weighted, but cells outside $a_k$ are still driven toward their occupancy label rather than toward $1/2$. Beliefs about cells an observation never sensed are therefore confident, pass the gate, and enter the average. The geometric weight is then the only remaining filter, and it reflects proximity to the line-of-sight segment rather than the propagation-relevant set: reflected paths reach cells far from that segment, while cells near it need not interact with the link at all.

\section{Evaluation metrics}
\label{app:metrics}

For a reconstruction $\hat{\mathbf{s}}$ and ground truth $\mathbf{s}$, let $\mathrm{TP}$, $\mathrm{FP}$, $\mathrm{TN}$, and $\mathrm{FN}$ count the cells in each outcome, with occupied cells as the positive class. We report
\begin{equation} \mathrm{BA} =\frac12\left(\frac{\mathrm{TP}}{\mathrm{TP}+\mathrm{FN}} + \frac{\mathrm{TN}}{\mathrm{TN}+\mathrm{FP}}\right), \; \mathrm{F1} = \frac{2\,\mathrm{TP}}{2\,\mathrm{TP}+\mathrm{FP}+\mathrm{FN}}, \; \mathrm{IoU} = \frac{\mathrm{TP}}{\mathrm{TP}+\mathrm{FP}+\mathrm{FN}}. \label{eq:metrics} \end{equation}
Balanced accuracy weights the two classes equally, so a constant or random predictor scores $50\%$ irrespective of the fraction of occupied cells; this makes it the appropriate indicator of chance-level behavior. F1 and IoU depend only on the occupied class and are insensitive to correctly predicted free space, which dominates most scenes, and are therefore stricter measures of footprint recovery. For a single scene the two are related by $\mathrm{IoU} = \mathrm{F1}/(2-\mathrm{F1})$; they differ after averaging over scenes and are both reported for comparability with prior work. All metrics are computed per scene over the full $N\times N$ grid and averaged over the 671 test scenes.

\section{Physical Interpretation of Pretrained Representations}
\label{app:tsne}

The hierarchical design of Section~\ref{sec:method} rests on the claim that the \texttt{[PHYSC]} token summarizes the physical structure that produced an observation, rather than the channel as a signal. This appendix examines that claim directly, independently of the reconstruction task, by inspecting the geometry of the representation space itself.

\textbf{Setup.} We encode a set of held-out observations with the frozen foundation model and project the resulting \texttt{[PHYSC]} embeddings to two dimensions with t-SNE~\citep{maaten2008visualizing}, which preserves local neighborhood structure. As a control, we apply the identical projection to the raw channel measurements from which those embeddings were computed. Any structure visible in the embeddings but absent from the raw channels is structure the pretraining objective has introduced, since no information is created by the encoder (Section~\ref{sec:probabilistic}). We color the same two projections by three properties, each relevant to a different aspect of the propagation problem: the spatial cluster of the user position, the dominant propagation direction of the link, and the identity of the environment in which the measurement was taken.

\textbf{Observations.} Figure~\ref{fig:tsne} shows the projections. Under all three colorings, the \texttt{[PHYSC]} embeddings form markedly clearer structure than the raw channels, and each of the three carries a distinct implication for this work.

Separation by \emph{user position} indicates that the representation resolves where a link terminates, which is a prerequisite for placing the evidence an observation carries at the correct location in the scene. Separation by \emph{dominant propagation direction} indicates that the representation disentangles the angular components of the channel --- the departure and arrival directions that determine which structures a link interacts with --- rather than retaining the raw superposition in which those components are mixed. Separation by \emph{scene identity} is the most directly relevant here: observations taken in the same environment are encoded to nearby points, even though the users occupy different positions and therefore observe different channels. The representation thus carries information about the environment that produced a measurement and not merely about the measurement itself, which is what makes a set of embeddings from one scene an informative conditioning set for the posterior over that scene.
\begin{figure*}[t]
    \centering
    \includegraphics[width=1\textwidth]{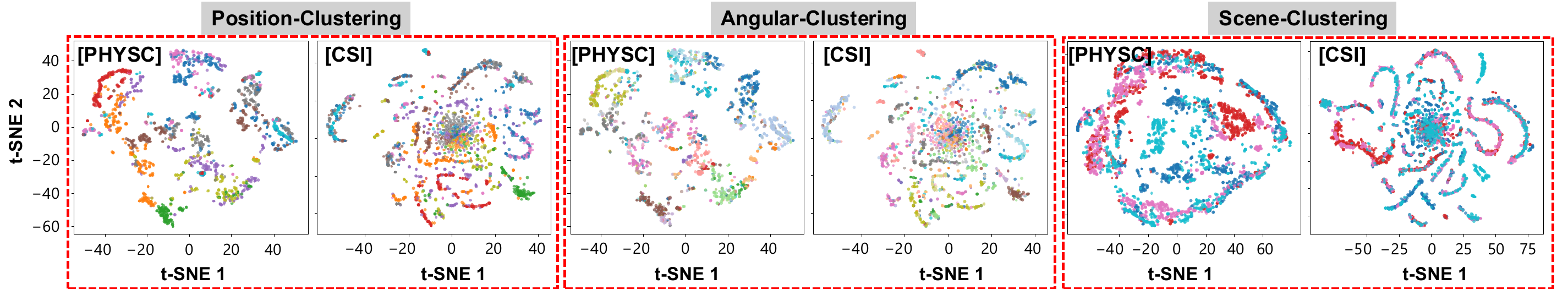}
    \vspace{-5mm}
    \caption{t-SNE of the encoded [PHYSC] token and raw CSI samples, colored by spatial cluster ID, dominant propagation path, and scene identity, showing foundation model representation forms clearer clustering across physically-relevant wireless link properties.}
    \label{fig:tsne}
    \vspace{-4mm}
\end{figure*}

\textbf{Relation to the conditioning argument.} The raw channels contain at least as much information as the embeddings computed from them, so the contrast in Figure~\ref{fig:tsne} is not one of information content but of accessibility. In the raw space, observations that share an environment are not neighbors, and a decoder conditioned there must learn the mapping from channel structure to scene structure before it can exploit the conditioning at all. In the representation space, that structure is already exposed. This is the geometric counterpart of the learnability argument of Section~\ref{sec:method}, and it is consistent with the behavior of the two conditioning schemes in Section~\ref{sec:experiment}: the diffusion decoder conditioned on raw measurements does not improve as observations accumulate, while the same decoder conditioned on \texttt{[PHYSC]} embeddings does.

\vspace{-4mm}
\section{Availability of Position Information}
\label{app:position}
\vspace{-2mm}
Each wireless observation is represented as a channel measurement and the corresponding user position, $(H_k,\mathbf{p}_k)$. To evaluate the effect of incomplete position information, we vary the fraction of measurements with available position labels from $10\%$ to $100\%$, while keeping the total number of measurements per scene fixed $K=500$. Figure~\ref{fig:acc_vs_fraction} reports the resulting balanced accuracy.

\begin{wrapfigure}{r}{0.43\textwidth}
\vspace{-2mm}
    \centering
    \includegraphics[width=\linewidth]{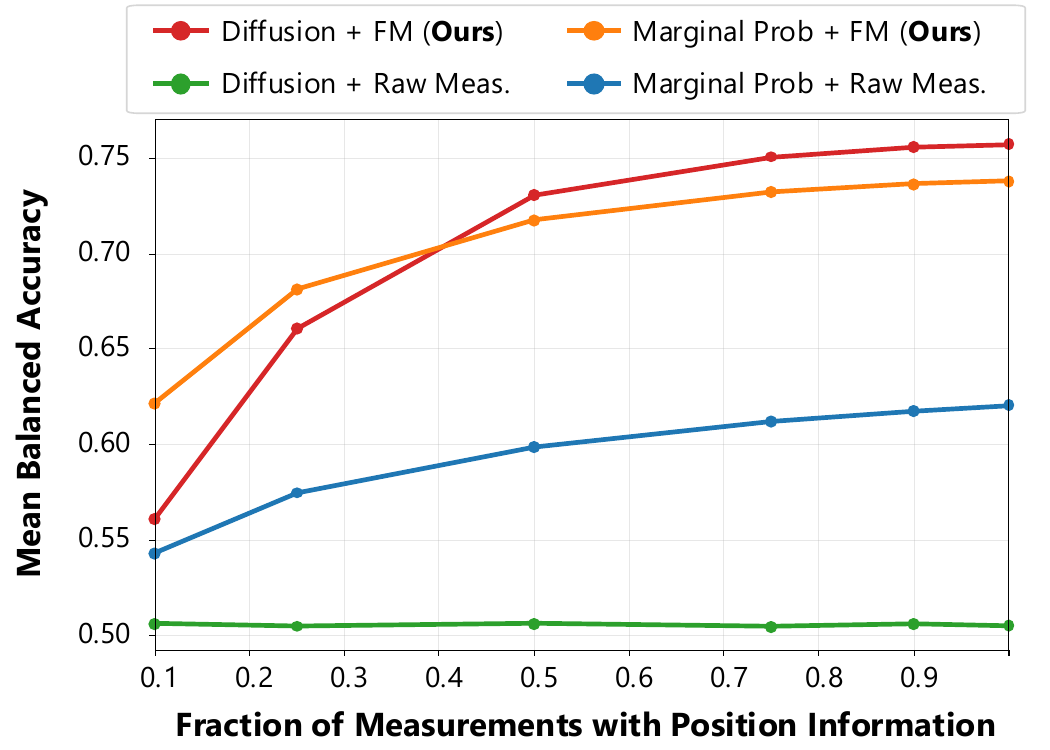}
    \vspace{-7mm}
    \caption{Mean scene reconstruction accuracy as a function of the fraction of measurements with position information.}
    \label{fig:acc_vs_fraction}
    \vspace{-4mm}
\end{wrapfigure}

All models except \textit{Diffusion + Raw Meas.} improve as position availability increases and eventually saturate. \textit{Diffusion + FM} improves from $0.56$ at $10\%$ availability to $0.76$ at $100\%$, while \textit{Marginal Prob. + FM} increases from $0.62$ to $0.74$ and \textit{Marginal Prob. + Raw Meas.} from $0.54$ to $0.62$. In contrast, \textit{Diffusion + Raw Meas.} remains near $0.50$ throughout.

Position information is essential because a channel measurement reveals propagation-induced geometric information but does not by itself anchor that information to an absolute location in the scene. The user position maps the propagation evidence to specific scene cells, enabling the model to spatially localize the structures associated with each observation. As a result, reducing position availability directly reduces the amount of spatially anchored evidence available for reconstruction.

The performance curves largely saturate once roughly half of the measurements have position information, indicating that additional labeled observations provide diminishing benefits after sufficient spatial coverage has been established. At low availability, \textit{Marginal Prob. + FM} initially outperforms \textit{Diffusion + FM}, whereas the ordering reverses as more position-labeled observations become available. This behavior reflects the different inference mechanisms: the marginal formulation applies an explicit geometric fusion rule that remains effective with sparse anchored evidence, while the diffusion model benefits more from a sufficiently rich set of observations that it can jointly integrate through the learned posterior.

Finally, the near-constant performance of \textit{Diffusion + Raw Meas.} across the entire sweep further indicates that its poor reconstruction quality is not primarily caused by insufficient position information, but by its inability to effectively exploit the raw channel measurements as conditioning information.

\section{Robustness to Noisy Position Information}
\label{app:noisy_pos}

Appendix~\ref{app:position} studies the case where position labels are missing for a fraction of observations. In deployment, another non-ideality can be that every observation carries a position, but the position is inaccurate. We therefore evaluate all models with perturbed user positions $\tilde{\mathbf{p}}_k = \mathbf{p}_k + \mathbf{e}_k$, $\mathbf{e}_k \sim \mathcal{N}(\mathbf{0}, \sigma^2 \mathbf{I}_2)$, with $\sigma = 5$\,m per horizontal axis, drawn independently for each observation. All models are trained on exact positions; noise is applied only at test time, with $K = 500$ observations per scene and all other settings as in Section~\ref{sec:experiment}. Under this model the radial error $\|\mathbf{e}_k\|$ is Rayleigh distributed, with median $\sigma\sqrt{2\ln 2} \approx 5.9$\,m, mean $\sigma\sqrt{\pi/2} \approx 6.3$\,m, and 95th percentile $\approx 12.2$\,m. This range is representative of consumer GNSS in outdoor urban settings \citep{merry2019smartphone,palathinkal2026evaluating}.

\begin{table}[h]
\centering
\caption{Reconstruction under noisy user positions ($\sigma = 5$\,m per axis, $K=500$). Values in parentheses give the relative change of the mean with respect to exact positions (Table~1).}
\label{tab:noisy_pos}
\resizebox{\linewidth}{!}{
\begin{tabular}{l ccc ccc ccc}
\toprule
& \multicolumn{3}{c}{Balanced Accuracy $\uparrow$} & \multicolumn{3}{c}{F1 $\uparrow$} & \multicolumn{3}{c}{IoU $\uparrow$} \\
\cmidrule(lr){2-4}\cmidrule(lr){5-7}\cmidrule(lr){8-10}
Model & Mean & Median & P90 & Mean & Median & P90 & Mean & Median & P90 \\
\midrule
Diffusion + Raw Meas. & 0.505 ($\pm$0.0\%) & 0.500 & 0.517 & 0.053 (-1.5\%) & 0.018 & 0.150 & 0.030 ($\pm$0.0\%) & 0.009 & 0.081 \\
Marginal Prob. + Raw Meas. & 0.609 (-1.9\%) & 0.604 & 0.655 & 0.374 (-5.9\%) & 0.371 & 0.480 & 0.233 (-7.1\%) & 0.228 & 0.315 \\
Marginal Prob. + FM (Ours) & \textbf{0.728} (-1.4\%) & \textbf{0.725} & 0.797 & \textbf{0.617} (-2.5\%) & \textbf{0.619} & 0.714 & \textbf{0.453} (-3.7\%) & \textbf{0.449} & 0.556 \\
Diffusion + FM (Ours) & 0.696 (-8.2\%) & 0.676 & \textbf{0.829} & 0.553 (-16.5\%) & 0.538 & \textbf{0.763} & 0.397 (-21.6\%) & 0.368 & \textbf{0.617} \\
\bottomrule
\end{tabular}}
\end{table}

\paragraph{Results.}
Table~\ref{tab:noisy_pos} shows that both of our designed pipelines remain far above the baselines under position information non-idealities. It is worth noting that the diffusion decoder, treats each embedding as precise evidence about a specific location and has never seen mislocalized evidence during training. Observations whose evidence is shifted in different directions then conflict. Posterior samples disagree on the affected cells, and their marginals move toward $1/2$, and the consensus readout suppresses footprints that are otherwise recovered. This mechanism mainly removes true positives, which is consistent with the larger drop in F1 and IoU than in balanced accuracy.

The sensitivity of Diffusion + FM is a train--test mismatch rather than a limitation of the posterior formulation, and it can be addressed directly. The position error/incompleteness can be treated as part of the observation model by training the decoder on embeddings computed from perturbed/missing positions, while keeping the clean occupancy $\mathbf{s}$ as the target. The decoder then learns $p_\theta(\mathbf{s} \mid \tilde{\mathcal{Z}})$ under such conditions. The architecture and training objective will be unchanged. We leave this implementation for future work. 

\newpage
\section{Visual Demonstration of Scene Reconstruction All Models}
\label{app:visual_scene_recons}

\begin{figure*}[h]
    \centering
    \includegraphics[width=1\textwidth]{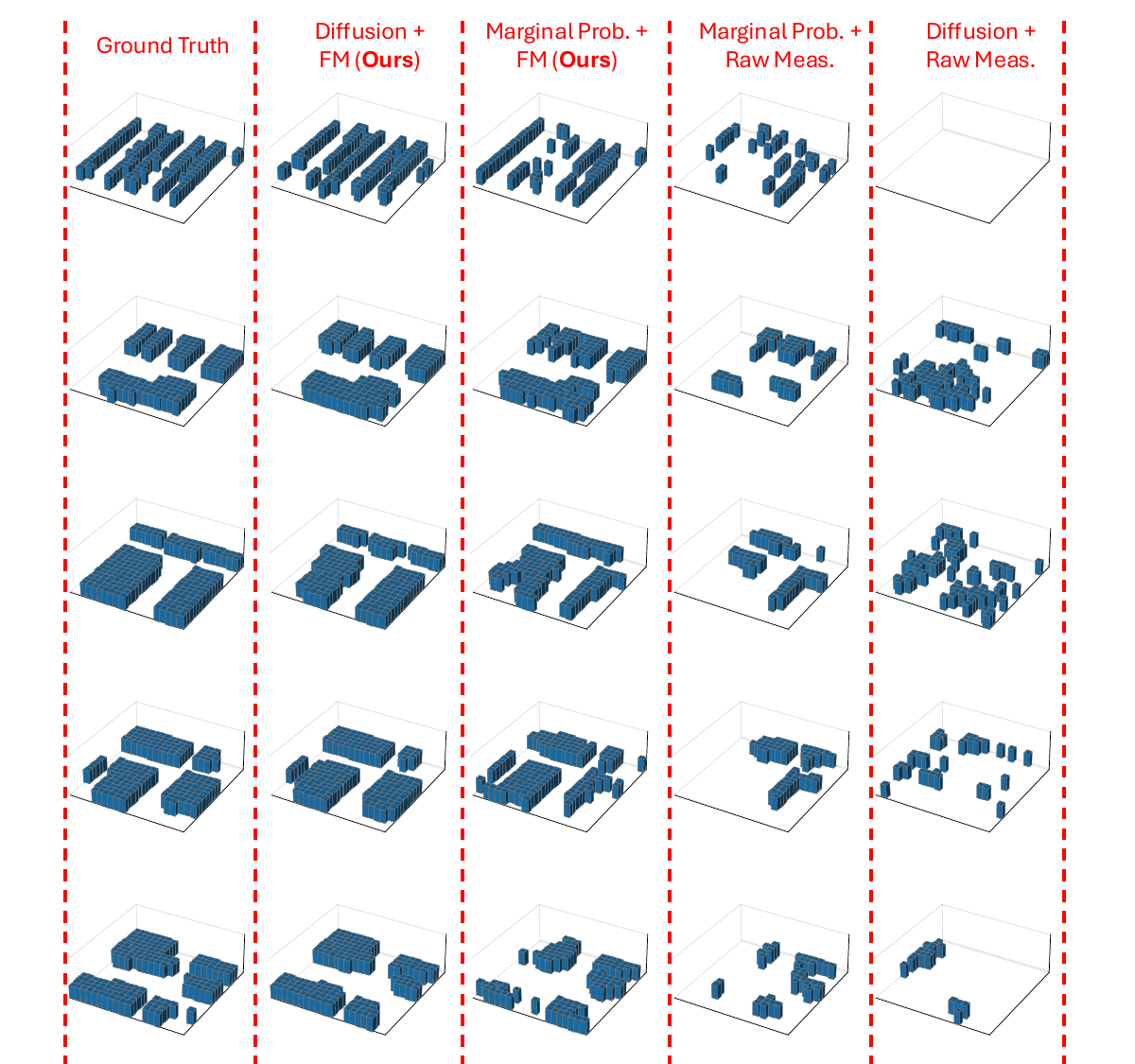}
    \vspace{-5mm}
    \caption{Visual comparison of scene reconstructions produced by all models against the ground-truth scenes.}
    \label{fig:vis_scene_all}
    \vspace{-4mm}
\end{figure*}

\begin{figure*}[h]
    \centering
    \includegraphics[width=1\textwidth]{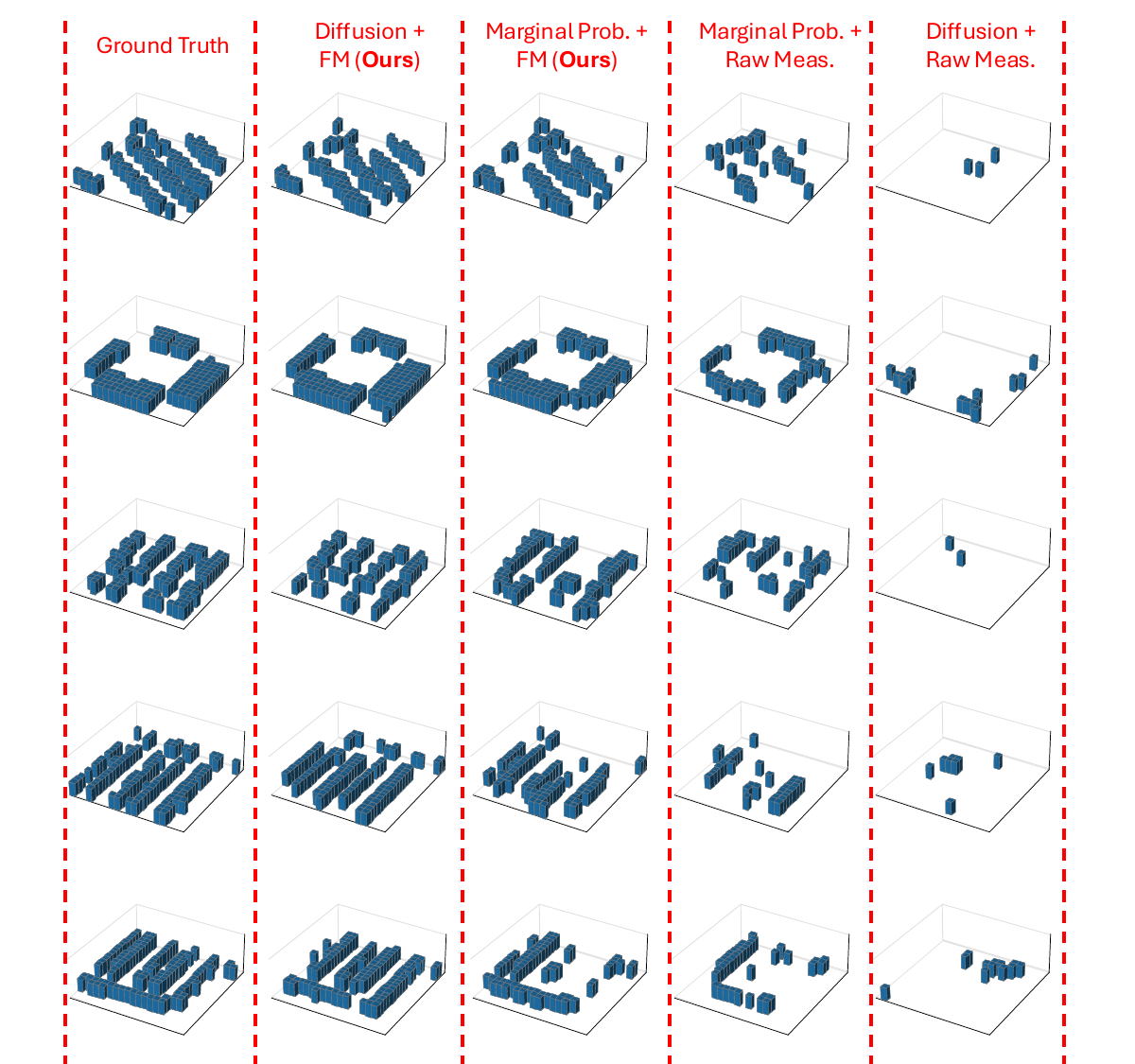}
    \vspace{-5mm}
    \caption{Visual comparison of scene reconstructions produced by all models against the ground-truth scenes.}
    \label{fig:vis_scene_all_2}
    \vspace{-4mm}
\end{figure*}

\end{document}

%% file: math_commands.tex
\usepackage{amsmath,amsfonts,bm}

\def\eqref#1{equation~\ref{#1}}

\def\1{\bm{1}}

\DeclareMathAlphabet{\mathsfit}{\encodingdefault}{\sfdefault}{m}{sl}
\SetMathAlphabet{\mathsfit}{bold}{\encodingdefault}{\sfdefault}{bx}{n}

